\documentclass[10pt,twocolumn,letterpaper]{article}

\usepackage[pagenumbers]{cvpr} 

\usepackage{graphicx}
\usepackage{amsmath}
\usepackage{amssymb}
\usepackage{booktabs}

\usepackage{times}
\usepackage{epsfig}

\usepackage[pagebackref,breaklinks,colorlinks]{hyperref}

\usepackage[capitalize]{cleveref}
\crefname{section}{Sec.}{Secs.}
\Crefname{section}{Section}{Sections}
\Crefname{table}{Table}{Tables}
\crefname{table}{Tab.}{Tabs.}

\usepackage{cite}
\usepackage{soul}
\usepackage{url}
\usepackage{amsthm}
\usepackage{amssymb}
\usepackage{pifont}
\usepackage{algorithm}
\usepackage{algpseudocode}
\usepackage{caption}
\usepackage[table, dvipsnames]{xcolor}
\usepackage{amsfonts}
\usepackage{tabularx}
\usepackage{enumitem}
\usepackage{multirow}

\definecolor{LightBlue}{rgb}{0.86,0.90,0.95}
\definecolor{LightGreen}{rgb}{0.86,0.95,0.90}

\newcommand{\cmark}{\ding{51}}%
\newcommand{\xmark}{\ding{55}}%

\makeatletter
\def\thm@space@setup{\thm@preskip=0pt
\thm@postskip=0pt}
\makeatother
\newtheorem{proposition}{Proposition}

\newtheorem{lemma}{Lemma}

\newtheorem{assumption}{Assumption}
\Crefname{assumption}{Assumption}{Assumptions}

\newcommand{\sysname}{MULE}
\newcommand{\sysnameacor}{ACO-R}
\newcommand{\sysnamemenn}{Beta-ENN}
\newcommand{\sysnameedc}{M-EDC}
\newcommand{\indep}{\perp \!\!\! \perp}

\makeatletter
\newcommand{\multiline}[1]{%
  \begin{tabularx}{\dimexpr\linewidth-\ALG@thistlm}[t]{@{}X@{}}
    #1
  \end{tabularx}
}
\makeatother

\def\cvprPaperID{12} 
\def\confName{CVPR}
\def\confYear{2023}

\begin{document}

\title{Open Set Action Recognition via Multi-Label Evidential Learning}

\author{Chen Zhao, Dawei Du, Anthony Hoogs, Christopher Funk\\
Kitware Inc.\\
{\tt\small $\{$chen.zhao, dawei.du, anthony.hoogs, christopher.funk$\}$@kitware.com}
}
\maketitle

\setlength{\abovedisplayskip}{3pt}
\setlength{\belowdisplayskip}{3pt}

\begin{abstract}
   Existing methods for open set action recognition focus on novelty detection that assumes video clips show a single action, which is unrealistic in the real world.
We propose a new method for open set action recognition and novelty detection via \textbf{MU}lti-\textbf{L}abel \textbf{E}vidential learning (\sysname{}), that goes beyond previous novel action detection methods by addressing the more general problems of single or multiple actors in the same scene, with simultaneous action(s) by any actor. 
Our Beta Evidential Neural Network estimates multi-action uncertainty with Beta densities based on actor-context-object relation representations. An evidence debiasing constraint is added to the objective function for optimization to reduce the static bias of video representations, which can incorrectly correlate predictions and static cues. We develop a learning algorithm based on a primal-dual average scheme update to optimize the proposed problem. Theoretical analysis of the optimization algorithm demonstrates the convergence of the primal solution sequence and bounds for both the loss function and the debiasing constraint. Uncertainty and belief-based novelty estimation mechanisms are formulated to detect novel actions.
Extensive experiments on two real-world video datasets show that our proposed approach achieves promising performance in single/multi-actor, single/multi-action settings.
\end{abstract}

\section{Introduction}
    \label{sec:intro}
Open set human action recognition has been studied in recent years due to its great potential in real-world applications, such as security surveillance~\cite{aggarwal2011human}, autonomous driving~\cite{roitberg2020open}, and face recognition~\cite{liu2017sphereface}. 
It differs from closed set problems that aim to classify human actions into a predefined set of known classes, since open set methods can identify samples with unseen classes with high accuracy~\cite{geng2020recent}.

\begin{figure}[t]
    \centering
    \includegraphics[width=\linewidth]{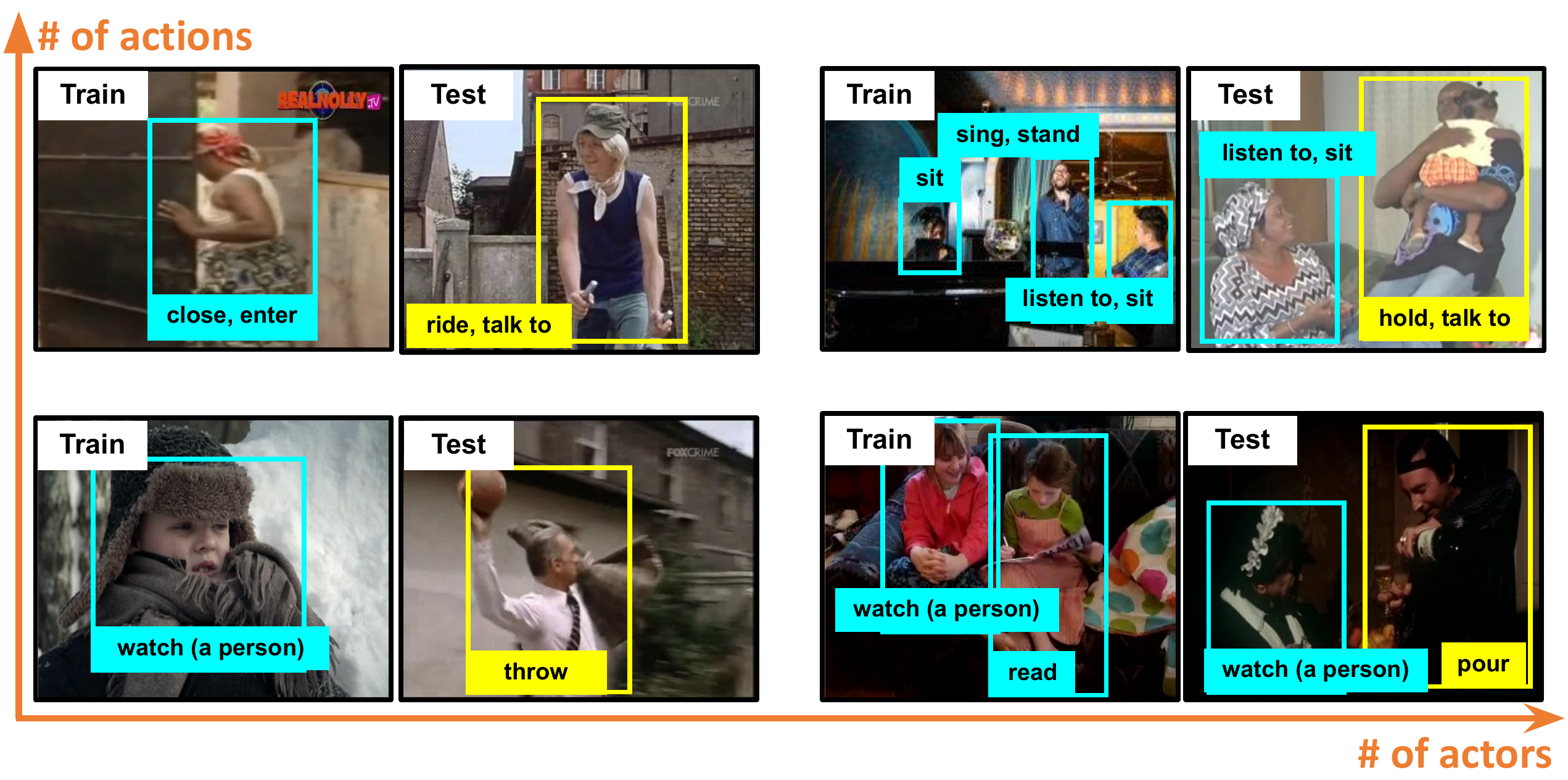}
    \vspace{-5mm}
    \caption{Novelty detection examples of single/multiple actor(s) with single/multiple action(s) in video~\cite{gu-2018-cvpr-ava,sigurdsson2016hollywood_Charades}, where an actor is identified as novel (\textcolor{Goldenrod!80!black}{yellow}) rather than being from a known category (\textcolor{cyan}{cyan}) in inference. Existing works~\cite{bao-2021-ICCV,bendale2016towards} on open set action recognition focus on single actor associated with single action (bottom-left), while our method can handle different situations.} 
    \vspace{-3mm}
    \label{fig:example}
\end{figure}    

To this end, several recent methods~\cite{bao-2021-ICCV,bendale2016towards} are proposed for open set human action recognition. 
As shown in the bottom-left of \Cref{fig:example}, they focus on single-actor, single-action based recognition, assuming that each video contains only one single action. 
Compared with traditional softmax scores~\cite{gal-2016-icml,liang2017enhancing,liu2020energy} for closed set recognition, evidential neural networks (ENNs)~\cite{sensoy-2018-nips,bao-2021-ICCV} can provide a principled way to jointly formulate the multi-class classification and uncertainty modeling to measure novelty of an instance more accurately. It assumes that class probability follows a prior Dirichlet distribution. 
However, in more realistic situation with multiple actions of actor(s) (see the upper part of \Cref{fig:example}), the Dirichlet distribution does not hold because the predicted likelihood of each action follows a binomial distribution (\ie, identifying either known or novel action).

In this paper, we introduce a general but understudied problem, namely \textit{novelty detection of actor(s) with multiple actions}. 
Given real-world use cases~\cite{geng2020recent,sozykin2018multi}, the goal is to accurately detect if actor(s) perform novel/unknown action(s) or not. Following~\cite{DBLP:conf/nips/WangLBL21}, an actor is considered unknown if it does not contain any known action(s).
Inspired by the belief theory~\cite{yager-2008-classic,josang-2016-subjective}, we propose a new framework named \textit{\textbf{MU}lti-\textbf{L}abel \textbf{E}vidential learning} (\sysname{}), which is composed of three modules: Actor-Context-Object Relation modeling (\sysnameacor{}), Beta Evidential Neural Network (\sysnamemenn{}), and Multi-label Evidence Debiasing Constraint (\sysnameedc{}).
First, we build \sysnameacor{} representation to exploit the actors' interactions with the surrounding objects and the context.
Then, we use \sysnamemenn{} to estimate the evidence of known actions, and quantify the predictive uncertainty of actions so that unknown actions would incur high uncertainty, \ie, lack of confidence for known predictions.
Here, the evidence indicates actions closest to the predicted one in the feature space and are used to support the decision-making~\cite{sensoy-2018-nips}. 
Instead of relying on Dirichlet distribution~\cite{bao-2021-ICCV}, the evidence in \sysnamemenn{} is regarded as parameters of a Beta distribution which is a conjugate prior of the Binomial likelihood. 

Additionally, in open set recognition, static bias~\cite{li-2018-eccv} may bring a false correlation between the prediction and static cues, such as scenes, resulting in inferior generalization capability of a model.
Therefore, the \sysnameedc{} is added to the objective function of our framework to reduce the static bias for video actions.
We propose a duality-based learning algorithm to optimize the network. 
Specifically, we apply an averaging scheme to proximate primal optimal solutions. 
The primal and dual parameters are updated interactively, where the primal parameters regard model accuracy and dual parameters adjust model debiasing.
The theoretical analysis shows the convergence of the primal solution sequence and gives bounds for both the loss function and the violation of the debiasing constraint in \sysname{}. 
To adapt to our proposed problem, we re-split two representative action recognition datasets (\ie, AVA~\cite{gu-2018-cvpr-ava} and Charades~\cite{sigurdsson2016hollywood_Charades}) into subsets of known actions and novel actions. 
According to the proposed uncertainty and belief based novelty estimation mechanisms, our model outperforms the state-of-the-art on novelty detection. 
The main contributions of this work are summarized:
\begin{itemize}[leftmargin=*,topsep=0pt,itemsep=1ex,partopsep=1ex,parsep=0ex]
    \item A new framework \sysname{} is proposed for open set action recognition in videos. To the best of our knowledge, this is the first study to detect actors with multiple unknown actions. Furthermore, our method can generalize to scenarios where a video contains either a single or multiple actors associated with one or more actions.
    \item To optimize the \sysnamemenn{}, we develop a simple but effective algorithm using a primal-dual average scheme update, with theoretical guarantees on the convergence of the primal solution sequence and bounds for both the loss function and the violation of the debiasing constraint.
    \item To estimate the novelty score for each actor, we introduce four novelty estimation mechanisms. Extensive experiments demonstrate that our proposed \sysname{} outperforms existing methods on novel action detection. 
\end{itemize}  

\begin{figure*}[t]
    \centering
    \includegraphics[width=0.9\linewidth]{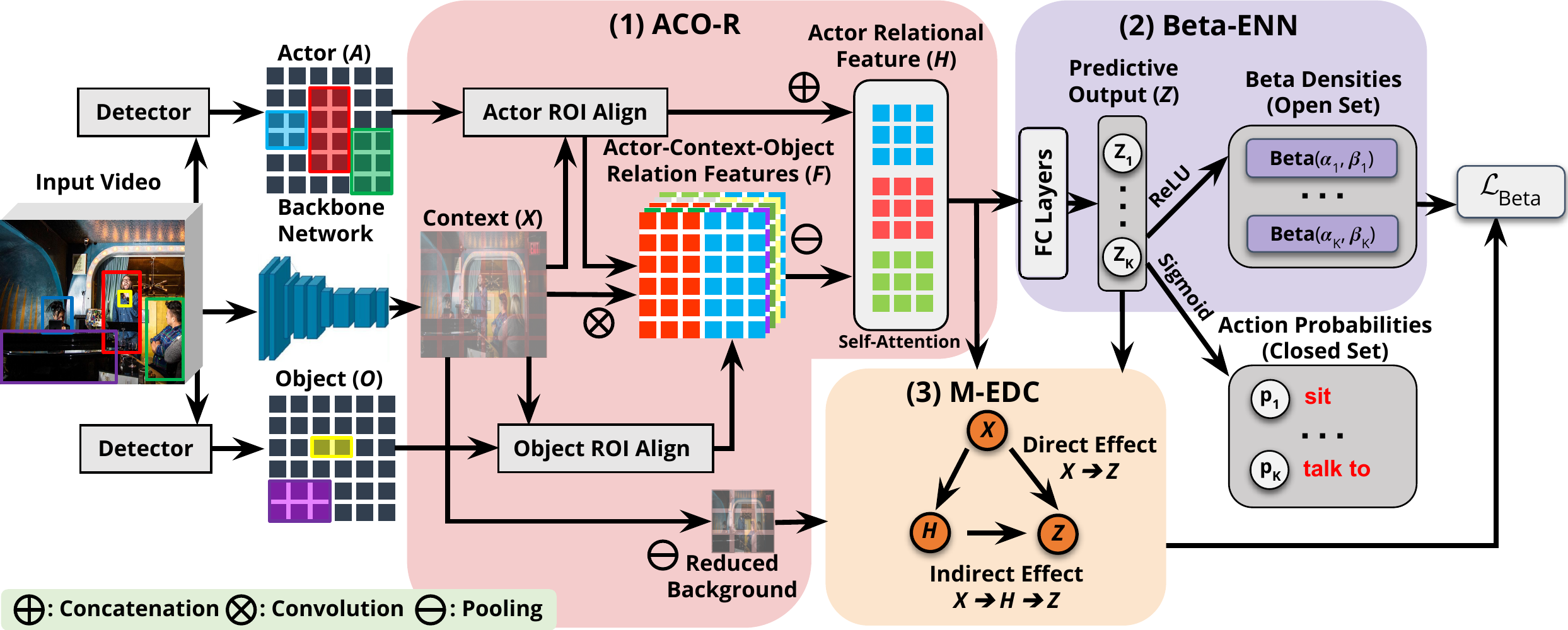}
    \caption{Our \sysname{} framework is composed of three modules. (1) Actor relational features are extracted through the \sysnameacor{} module, where they encode information based on interactions between actor/object instances with the context. (2) Positive and negative evidence are estimated through \sysnamemenn{} to quantify the predictive uncertainty of various human actions. (3) \sysnameedc{} is a debiasing constraint added to the loss function $\mathcal{L}_{beta}$ in \Cref{eq:opt-problem}, which aims to mitigate static bias.}
    \label{fig:arch}
    \vspace{-5mm}  
\end{figure*}
    
\section{Related Work}
    \label{sec:related work}
    \textbf{Open set action recognition.} 
Despite a great many explorations on video action recognition~\cite{DBLP:conf/iccv/Feichtenhofer0M19,feichtenhofer2020x3d,zhang2021multi}, the majority of existing methods are developed under the assumption that all actions are known a priori (closed set), and limited attention is given to the open set problems. 
OpenMax~\cite{bendale2016towards} is initially proposed for open set recognition, in which it leverages \textit{Extreme Value Theory}~\cite{DBLP:journals/pami/RuddJSB18} to expand the $K$-class softmax classifier. 
Roitberg~\etal~\cite{DBLP:conf/bmvc/RoitbergAS18} develops the voting-based scheme to leverage the estimated uncertainty of action predictions to measure the novelty of the test sample.
DEAR~\cite{bao-2021-ICCV} formulates the open set action recognition problem by estimating the uncertainty of single labeled actions to distinguish between the known and unknown samples. 
Fang~\etal~\cite{fang2021learning} theoretically prove the learnability and the generalization bound for an open set recognition classifier. 
However, such methods focus on simple scenarios where each actor has only one action in a video.

\textbf{Uncertainty estimation.} 
To distinguish unknown samples from known ones, how deep neural networks identify samples belonging to an unrelated data distribution becomes crucial. To this end, a stream of research on Bayesian Neural Nets (BNNs)~\cite{kendall2017uncertainties} is proposed to estimate prediction uncertainty by approximating the moments of the posterior predictive distribution. However, BNNs face several limitations, including the intractability of directly inferring the posterior distribution of the weights given data, the requirement and computational expense of sampling during inference, and the question of how to choose a weight prior~\cite{NEURIPS2020_aab08546}. 
Evidential Neural Networks (ENNs)~\cite{sensoy-2018-nips} have recently been proposed to estimate evidential uncertainty for multi-class classification problems. But they are designed for single-label multi-classification issues by assuming that class probability follows a prior Dirichlet distribution. 

\textbf{Debiasing.} 
Static bias is another challenging issue limiting the generalization capability of a model in an open-set setting~\cite{bao-2021-ICCV}. 
The manifestation of static bias can often be as fraught as the spurious correlation between the prediction and sensitive features like unrelated objects and background~\cite{li-2018-eccv, zhao2020fair}.
Existing works~\cite{li-2018-eccv,bao-2021-ICCV} empirically show that debiasing the model by input data or learned representation can improve the action recognition accuracy.
RESOUND~\cite{li-2018-eccv} indicates that static bias may help to achieve better results in a closed-set setting if an action can overfit it. 
On the contrary, for an open-set setting, DEAR~\cite{bao-2021-ICCV} states the static bias could result in a vulnerable model and further introduces the contrastive evidence debiasing by temporally shuffled feature input and 2D convolution. 
However, it still pushes the sensitive feature to be independent of the non-sensitive one only, in which the dependency of the sensitive feature on model predictions is ignored.       

To address the limitation of the above methods, we propose a new Beta distribution based network with multi-label evidential learning, where static bias is reduced in terms of both indirect and direct dependencies. To our knowledge, it is the first time to solve open set action recognition based on single or multiple actor(s) and action(s).
    
\section{Method}
    \label{sec:method}
    Given an input video, we aim to determine the predictive uncertainty of the detected actors with unknown actions.
For open set inference, actions with high and low uncertainty can be regarded as unknown and known, respectively. 
As illustrated in \Cref{fig:arch}, our method extracts Actor-Context-Object Relation (\sysnameacor{}) to train the Beta Evidential Neural Network (\sysnamemenn{}) with the Multi-label Evidence Debiasing Constraint (\sysnameedc{}). In the following, we describe each module in detail.

\subsection{Actor-Context-Object Relation Modeling}
\label{sec:ACOR}
Similar to~\cite{sun-2018-ECCV,pan2020actorcontextactor}, we extract rich information (\ie, actors, objects, and context) from videos to predict actions in complex scenes.
To reduce complexity, we can detect $N$ actors and $M$ objects for each input video clip based on the key frames.  
Based on the off-the-shelf Faster R-CNN detector~\cite{ren-2015-faster}, actor features $\{A^i\}_{i=1}^N\in\mathbb{R}^C$ and object features $\{O^i\}_{i=1}^M\in\mathbb{R}^C$ are produced by applying RoIAlign~\cite{he-2017-ICCV} followed by spatial max pooling, each of which describes the spatio-temporal appearance and motion of one Region of Interest (RoI). 
In addition, the SlowFast~\cite{DBLP:conf/iccv/Feichtenhofer0M19} backbone extracts a spatio-temporal feature volume from the input video. 
Then we can obtain the context feature $X\in\mathbb{R}^{C\times H \times W}$ by performing average pooling along the temporal dimension, where $C,H,W$ represent channel, height, and width. 

To model actor-context-object relations, we first concatenate all actor features to the context feature followed by object features to form a series of concatenated feature maps $\{F^i\}_{i=1}^{N\times M}\in\mathbb{R}^{3C\times H \times W}$. 
Following a stack of transformer blocks~\cite{pan2020actorcontextactor}, we can obtain actor relational features $\{H^i\}_{i=1}^{N\times M}$ by calculating the higher-order relations between pairs of $\{F^i\}_{i=1}^{N\times M}$ at the same spatial location, where two actors can be associated via the same spatial context and different objects.

\subsection{The Beta Evidential Neural Network}
\label{sec:edl}
As discussed in \Cref{sec:intro}, existing models~\cite{gal-2016-icml,liang2017enhancing,liu2020energy} typically rely on a softmax layer to perform multi-class classification. Since the softmax score is essentially a point estimation of a predictive distribution~\cite{sensoy-2018-nips,bao-2021-ICCV,gal-2016-icml}, the models cannot estimate the predictive uncertainty of out-of-distribution. 
To tackle this limitation, evidential neural networks (ENNs)~\cite{sensoy-2018-nips} are developed to jointly formulate the multi-class classification and uncertainty modeling. 
Specifically, ENNs~\cite{NEURIPS2020_aab08546, bao-2021-ICCV} 1) interpret the standard output of a classification network as the parameter set of a categorical distribution and 2) assume that class probability follows a prior Dirichlet distribution and replaces this parameter set with the parameters of a Dirichlet density for novelty detection.

However, in our setting, the assumption of using Dirichlet distribution is not a good fit for an actor with multiple actions. It is because the predicted likelihood for each action class follows a binomial distribution whose conjugate prior is a Beta distribution rather than Dirichlet distribution. To this end, we design the \sysnamemenn{} to classify known and novel actions based on Beta distributions.

\textbf{Beta distribution-based subjective opinions}. 
According to the belief theory~\cite{yager-2008-classic}, it is more reasonable to predict subjective opinions rather than class probabilities in an open-set setting. We hence use the principles of evidential theory to quantify belief masses and uncertainty in the proposed \sysnamemenn{} through Subjective Logic (SL)~\cite{josang-2016-subjective}. 

In multi-label action recognition, since each action follows a given binomial opinion towards the proposition, the subjective opinion $\omega_i=(b_i,d_i,u_i,a_i)$ for an action $i\in\{1,\cdots K\}$ of an actor is expressed by two belief masses, \ie, belief $b_i\in[0,1]$ and disbelief $d_i\in[0,1]$, and one uncertainty mass $u_i\in[0,1]$, where $b_i+d_i+u_i=1$. 
The expected belief probability $p_i$ is defined as $p_i=b_i+a_i\cdot u_i$, where $a_i$ refers to a base rate representing prior knowledge without commitment, such as neither agree nor disagree.

According to~\cite{sensoy-2018-nips}, a binomial opinion $\omega_i$ of action follows a Beta probability density function (pdf) denoted $\text{Beta}(p_i|\alpha_i,\beta_i)$, where $p_i\in[0,1]$ represents the action assignment probabilities. The Beta pdf is characterized by parameters $\alpha_i$ and $\beta_i$, where $\alpha_i$ and $\beta_i$ are viewed as positive and negative evidence of the observed action $i$, respectively. 
The evidence indicates actions closest to the predicted ones in the feature space and is used to support the decision-making.
Action labels are the same for positive evidence, but different for negative evidence.
Specifically, for each action, the opinion $\omega_i$ is obtained based on its corresponding $\alpha_i$ and $\beta_i$ using the rule in SL:
\begin{align}
\small
\label{eq:element-wise-sl}
    b_i=\frac{\alpha_i-a_i W}{\alpha_i+\beta_i},\quad d_i=\frac{\beta_i-a_i W}{\alpha_i+\beta_i},\quad u_i=\frac{W}{\alpha_i+\beta_i},
\end{align}
where we set the non-informative prior weight $W=2$ and base rate $a_i=1$ for each binary action classification (known/unknown) empirically.

Therefore, a collection of evidence pairs $\{(\alpha_i,\beta_i)\}_{i=1}^K$ of an actor is estimated to quantify the predictive uncertainty. An actor with multiple unknown actions would incur high uncertainty $u$ with low belief $b$ (see \Cref{sec:novelty_estimation}). 

\textbf{Learning opinions through Beta loss}.
In practice, for an actor $j$, its positive and negative evidence are estimated $\boldsymbol{\alpha}_j=s(h(\mathbf{x}_j;\boldsymbol{\theta}))+1$ and $\boldsymbol{\beta}_j=s(h(\mathbf{x}_j;\boldsymbol{\theta}))+1$, where $\boldsymbol{\alpha}_j=[\alpha_{1j},\cdots,\alpha_{Kj}]^T$ and $\boldsymbol{\beta}_j=[\beta_{1j},\cdots,\beta_{Kj}]^T$. 
$\mathbf{x}_j$ denotes the input video, $h(\mathbf{x}_j;\boldsymbol{\theta})$ represents the evidence vector predicted by the network for the classification and, $\boldsymbol{\theta}$ represents parameters for \sysnameacor~modeling.
Here $s(\cdot)$ is the evidence function (\textit{e.g.,} ReLU) to keep $\boldsymbol{\alpha}_j,\boldsymbol{\beta}_j\succcurlyeq \mathbf{1}$. 

To learn the above opinions, we define the Beta loss function by computing its Bayes risk for the action predictor. For the binary cross-entropy loss for each action $i$ over a batch of actors, the proposed Beta loss takes
\begin{align}
\label{eq:beta-loss}
    &\mathcal{L}_{Beta}(\boldsymbol{\theta})=\sum\nolimits_{j=1}^N\mathcal{L}_j(\boldsymbol{\theta}),
\end{align}
where $j\in\{1,\cdots,N\}$ denotes the index of an actor.
\begin{align}
\small
\label{eq:beta-loss-1-actor}
    \mathcal{L}_j(\boldsymbol{\theta}) \nonumber
    =&\sum\nolimits_{i=1}^K \int\textbf{BCE}(y_{ij},p_{ij})\textbf{Beta}(p_{ij};\alpha_{ij},\beta_{ij})dp_{ij} \nonumber\\
    =&\sum\nolimits_{i=1}^K \Big[y_{ij}\Big(\psi(\alpha_{ij}+\beta_{ij})-\psi(\alpha_{ij})\Big)\nonumber\\
    &+(1-y_{ij})\Big(\psi(\alpha_{ij}+\beta_{ij})-\psi(\beta_{ij})\Big)\Big], 
\end{align}
where $K$ is the number of actions, and $\textbf{BCE}(\cdot)$ is the binary cross-entropy loss, and $\psi(\cdot)$ is the \textit{digamma} function. The log expectation of Beta distribution derives the last equality. 
$\mathbf{y}_j=[y_{1j},\cdots,y_{Kj}]\in\{0,1\}^K$ is the $K$-dimensional ground-truth action(s) label for $\mathbf{x}_j$.

\subsection{Multi-Label Evidence Debiasing Constraint}
\label{sec:debias}
For open set action recognition, static bias~\cite{li-2018-eccv} could result in a vulnerable model that falsely recognizes an action containing similar static features.  
For example, the action of ``walking'' is easily recognized with ``road'' in the background, but it would be unable to recognize the same action with the ``treadmill'' scene.
From the perspective of fairness-aware learning~\cite{Zhao-ICDM-2019,zhao-KDD-2021}, as indicated in \Cref{fig:arch}, this is due to the spurious dependency of the predictive outcome $Z$ (\textit{e.g.,} actions) onto sensitive features $X$ (\textit{e.g.,} background scene), and strong dependency indicates strong effects.

To mitigate static bias, we introduce the evidence debiasing constraint in multi-label evidential learning. A fair prediction indicates no direct ($X\rightarrow Z$) or indirect ($X\rightarrow H \rightarrow Z$) dependency of $Z$ on the reduced $X$. These types of dependencies (direct and indirect) are supported by frameworks applied to large bodies of cases throughout statistical disparity~\cite{Barocas-CLR-2016}. 
Consequently, debiasing through both the direct and indirect effects enforces procedural fairness in decision-making by statistically mitigating the dependency of the sensitive feature $X$ on the prediction $Z$. It therefore guarantees outcome fairness among sensitive groups~\cite{Zhang-AAAI-2018}.

In particular, similar to~\cite{bahng-2020-ICML,bao-2021-ICCV}, the Hilbert-Schmidt Independence Criterion (HSIC) function measures the degree of independence between two continuous random variables. With radial basis function kernel $k_1$ and $k_2$, $\text{HSIC}^{k_1,k_2}(Z,\sigma(X))=0$ if and only if $Z\indep \sigma(X)$, where $\sigma(\cdot)$ is 2D average pooling operation. 
As shown in \Cref{fig:arch}, $Z\equiv h(\mathbf{x};\boldsymbol{\theta})$ represents the evidence vector predicted by the network, and $X$ indicates the context feature from the backbone. 
The debiasing constraint takes the form
\begin{align}
    g(\boldsymbol{\theta})\equiv \text{HSIC}\Big(h(\mathbf{x};\boldsymbol{\theta}),\sigma(X)\Big),
    \label{eq:constraint}
\end{align}
which aims to reduce both direct and indirect dependency of predictive outcomes onto background context.

\section{Optimization and Inference}
    \label{sec:algor}
    In summary, we combine the aforementioned \Cref{eq:beta-loss,eq:constraint} to formulate the optimization problem as
\begin{align}
\label{eq:opt-problem}
    \min_{\boldsymbol{\theta}\in\Theta}\quad \mathcal{L}_{Beta}(\boldsymbol{\theta}) \quad \text{subject to} \quad g(\boldsymbol{\theta})\leq \gamma,
\end{align}
where $\gamma>0$ is the independence criterion relaxation.
Previous methods, such as regularization or projection, can handle the constraint in \Cref{eq:opt-problem} with near-optimal solutions but do not directly provide primal solutions. It may even fail to produce any useful information for static bias. 
 
To better solve the problem, we develop the primal-dual average scheme update method. 
Specifically, we apply an averaging scheme to the primal sequence $\{\boldsymbol{\theta}^{(m)}\}_{m=1}^\infty$ to approximate primal optimal solutions, where $m$ represents the index of an iteration. In particular, the sequence $\{\Tilde{\boldsymbol{\theta}}^{(m)}\}_{m=1}^\infty$ is defined as the averages of the previous vectors through $\boldsymbol{\theta}^{(0)}$ to $\boldsymbol{\theta}^{(m-1)}$, \ie,
\begin{align}
\label{eq:primal-avg}
    \Tilde{\boldsymbol{\theta}}^{(m)} = \frac{1}{m} \sum\nolimits_{i=1}^{m-1}\boldsymbol{\theta}^{(i)}, \quad \forall m\geq 1.
\end{align}
The primal feasible iterate $\boldsymbol{\theta}^{(m)}$ is given in \Cref{eq:primal-solution}. 
To simplify, we abuse the notation $\mathcal{L}_{Beta}(\boldsymbol{\theta})$ with $\mathcal{L}(\boldsymbol{\theta})$ in the rest of the paper.
\begin{align}
\label{eq:primal-solution}
    \boldsymbol{\theta}^{(m)} \leftarrow \arg\min_{\boldsymbol{\theta}\in\Theta} &\Big\{\mathcal{L}(\boldsymbol{\theta}^{(m-1)}) + \lambda^{(m-1)}\Big(g(\boldsymbol{\theta}^{(m-1)})-\gamma\Big)\nonumber\\
    &-\frac{\delta}{2}\Big(\lambda^{(m-1)}\Big)^2\Big\},
\end{align}
where $\delta>0$ is a constant determined by analysis. Accordingly, the parameter $\lambda$ in dual solutions is updated as
\begin{align}
\small
\label{eq:dual-solution}
    \lambda^{(m)} \leftarrow \max\Big\{\Big[\lambda^{(m-1)}+\eta_2\Big(g(\Tilde{\boldsymbol{\theta}}^{(m)})-\gamma-\delta\lambda^{(m-1)}\Big)  \Big], 0\Big\},
\end{align}
where $\eta_2>0$ is a constant learning rate of the dual step.
By updating the dual parameter, $\lambda$, our optimization efficiently approaches the optimal model $\boldsymbol{\theta}^\ast$ arbitrarily close within a small finite number of steps $m$. 
For better understanding, the above algorithm is summarized in \Cref{alg:algor}.

\subsection{Theoretical Analysis}
To derive the bound on the feasibility violation and the primal cost of the running averages, we analyze the statistical guarantees of the solutions in \Cref{eq:primal-solution,eq:dual-solution}. We first make the following assumption.
\begin{assumption} (Regularity and Feasibility).
\label{assump1}
The convex set $\Theta$ is compact (\ie, closed and bounded). For any $\boldsymbol{\theta}\in\Theta$, $\mathcal{L}(\boldsymbol{\theta})$ and $g(\boldsymbol{\theta})$ are convex real-valued and bounded functions, where $\inf_{\boldsymbol{\theta}\in\Theta} g(\boldsymbol{\theta})=0$ and, for any $\boldsymbol{\theta}\notin \Theta$, \normalfont{dom}$(g(\boldsymbol{\theta}))=\emptyset$.
\end{assumption}
Recall that in the proposed \cref{alg:algor} used to approximate pairs of primal-dual parameters at each iteration $m$, for the averaged primal sequence $\{\Tilde{\boldsymbol{\theta}}^{(m)}\}$, we show that it always converges when $\Theta$ is compact.

\begin{proposition}(Convergence of Averaged Primal Sequence)
\label{prop1}
Under \Cref{assump1}, when the convex set $\Theta$ is compact, let the approximate primal sequence $\{\Tilde{\boldsymbol{\theta}}^{(m)}\}_{m=1}^{\infty}$ be the running averages of the primal iterates in \Cref{eq:primal-avg}. Then $\{\Tilde{\boldsymbol{\theta}}^{(m)}\}_{m=1}^{\infty}$ can converge to its limit $\Tilde{\boldsymbol{\theta}}^*$.
\end{proposition}
 
Next we provide bounds on the feasibility violation $g(\Tilde{\boldsymbol{\theta}}^{(m)})$ and the primal cost of the running averages $\mathcal{L}(\Tilde{\boldsymbol{\theta}}^{(m)})$, where the bounds are given per iteration $m$. 

\begin{proposition}(Bounds for $\mathcal{L}(\Tilde{\boldsymbol{\theta}}^{(m)})$ and the violation of $g(\Tilde{\boldsymbol{\theta}}^{(m)})$~\cite{averagedDS-SJO-2009})
\label{prop:prop2}
Let the dual sequence $\{\lambda^{(m)}\}_{m=1}^{\infty}$ be generated through \Cref{eq:dual-solution} and $\{\Tilde{\boldsymbol{\theta}}^{(m)}\}_{m=1}^{\infty}$ be the averages in \Cref{eq:primal-avg}. Under \cref{assump1}, we have
\begin{enumerate}[leftmargin=*,topsep=0pt,itemsep=0ex,partopsep=0ex,parsep=0ex]
    \item An upper bound on the amount of constraint violation of $\Tilde{\boldsymbol{\theta}}^{(m)}$ that $\big\lVert \big[g(\Tilde{\boldsymbol{\theta}}^{(m)})\big]_+\big\rVert \leq \frac{\lambda^{(m)}}{m\eta_2}$.
    \item An upper bound on $\mathcal{L}(\Tilde{\boldsymbol{\theta}}^{(m)})$ that $\mathcal{L}(\Tilde{\boldsymbol{\theta}}^{(m)}) \leq f^*+\frac{(\lambda^{(0)})^2}{2m\eta_2}+\frac{\eta_2 L^2}{2}$, where $\big\lVert g(\Tilde{\boldsymbol{\theta}}^{(m)})\big\rVert<L$ and $L>0$.
    \item A lower bound $\mathcal{L}(\Tilde{\boldsymbol{\theta}}^{(m)}) \geq f^*-\lambda^*\cdot\big\lVert \big[g(\Tilde{\boldsymbol{\theta}}^{(m)})\big]_+\big\rVert$.
\end{enumerate}
where $[u]_+$ denotes the projection of $[u]$ on the nonnegative orthant. $f^*$ is the optimal solution of \Cref{eq:opt-problem} and $\lambda^\ast$ denotes the optimal value of the dual variable.
\end{proposition}

\Cref{prop1,prop:prop2} demonstrate the convergence of the primal solution sequence and give bounds for both the loss function and debiasing constraint in \sysname{}. The detailed proof is given in the Appendix.


\begin{algorithm}[t!]
\small
\caption{Primal-Dual Average Scheme Update}
\label{alg:algor}
\textbf{Input}: $\boldsymbol{\theta}^{(0)}\in\Theta, \lambda^{(0)}\in\mathbb{R}_+$: primal and dual parameters\\
\textbf{Require}: $\eta_1,\eta_2> 0$: learning rates
\begin{algorithmic}[1]
\State Initialize an empty buffer $B=[\:]$ to store $\boldsymbol{\theta}^{(m)}$
    \For{$m=1,2,...$}
        \State \multiline{%
            $L(\boldsymbol{\theta},\lambda) := \mathcal{L}(\boldsymbol{\theta}) + \lambda(g(\boldsymbol{\theta})-\gamma)-\frac{\delta}{2}\lambda^2$ }
        \State \multiline{%
            \textbf{Primal Update:} \\ 
            $\boldsymbol{\theta}^{(m)} \leftarrow \text{Adam}\Big\{L\Big(\boldsymbol{\theta}^{(m-1)},\lambda^{(m-1)}\Big), \eta_1, \boldsymbol{\theta}^{(m-1)}\Big\}$ }
        \State Add $\boldsymbol{\theta}^{(m)}$ in $B$ 
        \State \multiline{%
            \textbf{Average Scheme:} 
            $\Tilde{\boldsymbol{\theta}}^{(m)} = \frac{1}{|B|}\sum_{i=1}^{|B|-1}\boldsymbol{\theta}^{(i)}$}
        \State Update $\boldsymbol{\theta}^{(m)}\leftarrow \Tilde{\boldsymbol{\theta}}^{(m)}$
        \State $L'(\boldsymbol{\theta},\lambda):=\lambda + \eta_2 \Big(g(\boldsymbol{\theta})-\gamma-\delta\lambda\Big)$
        \State \multiline{%
            \textbf{Dual Update:}\\
            $\lambda^{(m)} \leftarrow \max\Big\{L'\Big(\Tilde{\boldsymbol{\theta}}^{(m)},\lambda^{(m-1)}\Big), 0\Big\}$ }
    \EndFor
\end{algorithmic}
\end{algorithm}

\subsection{Novelty Score Estimation}
\label{sec:novelty_estimation}
During inference, we aim to detect novel actor(s) with single or multiple unknown action(s). According to \Cref{eq:element-wise-sl}, we develop four novelty quantification scores based on either uncertainty or belief of an actor.
To this end, we incorporate the actor's estimated subjective opinions $\{\omega_i\}_{i=1}^K$ for its actions, where $\omega_i=(b_i,d_i,u_i,a_i)$.

\textbf{Uncertainty-based novelty score.} 
As described in \Cref{sec:edl}, the predicted positive and negative evidence pair $\{(\alpha_i,\beta_i)\}_{i=1}^K$ are used to estimate uncertainty $u$ of an actor with $K$ actions. 
A value of $u$ close to $1$ indicates novelty. Three uncertainty-based novelty score estimation mechanisms are introduced by using positive (PE) or negative (NE) evidence only and aggregating them (PNE), \ie,
\begin{equation}
\small   %
\label{eq:novelty-scores}
\begin{aligned}
     \text{PE}: u =& \frac{2}{1+\exp({\sum\nolimits_i^K \alpha_i-K})}, \\
     \text{NE}: u =& \frac{2}{1+\exp({K-\sum\nolimits_i^K \beta_i})}-1, \\
    \text{PNE}: u =& \frac{2K}{\sum\nolimits_i^K (\alpha_i+\beta_i)}.
\end{aligned}
\end{equation}

\textbf{Belief-based novelty score.} 
Another novelty detection scheme is to estimate its belief value $b$ using the binomial co-multiplication operator (denoted as $\ast$)~\cite{Audun2006BeliefCalculus} for all actions, 
\begin{align}
    &b = b_1\ast\cdots\ast b_K, \\
    &\text{where } b_i * b_j := b_i+b_j-b_i\cdot b_j, \forall i,j\in\{1,\cdots K\}, i\neq j, \nonumber
\end{align}
$b_i\in[0,1]$ is a class-wise belief estimated using its corresponding positive and negative evidence values $\alpha_i$ and $\beta_i$ in \Cref{eq:element-wise-sl}. A belief value $b$ close to $0$ indicates novelty. 

\subsection{Relation with Existing Evidential Learning}
Although our method shares the basic concept of evidential learning with DEAR~\cite{bao-2021-ICCV}, it has a significant difference in three aspects.
\begin{itemize}[leftmargin=*,topsep=0pt,itemsep=0ex,partopsep=0ex,parsep=0ex]
    \item \textit{The Beta distribution in \Cref{eq:beta-loss} is generalized from the Dirichlet distribution in ENNs.} In other words, to detect an actor with multiple novel actions, we improve original ENNs~\cite{sensoy-2018-nips,bao-2021-ICCV} by using $K$ Beta distributions. If $K=1$, our loss function is reduced to the counterpart in DEAR~\cite{bao-2021-ICCV} (see the proof in the Appendix).
    \item \textit{The \sysnameedc{} in \Cref{eq:constraint} simultaneously mitigates direct and indirect bias.} In contrast, the contrastive evidence debiasing module in DEAR~\cite{bao-2021-ICCV} only considers dependencies of predictive outcome $Z$ on the sensitive feature $X$ through the causal path $X\rightarrow H$. This is viewed as a sub-path of the indirect dependency $X\rightarrow H\rightarrow Z$, resulting in inferior performance.
    \item \textit{Our optimization method in \Cref{alg:algor} provides optimal hyper-parameter search for robust learning.} DEAR~\cite{bao-2021-ICCV} views the debiasing constraint as a regularization term with the empirically set Lagrange multiplier.
    In contrast, an optimal multiplier $\lambda^\ast$ is automatically found as a dual variable by \Cref{alg:algor}. Iterative update between the primal and dual variable guarantees the achievement of a small duality gap.
\end{itemize}

\textbf{Discussion.} Dirichlet densities-based ENNs for open set action recognition focus on detecting actors with a single action. Such methods assume action probability follows a prior Dirichlet distribution. In contrast, in our \sysname{}, Beta distributions are more general to adapt to different applicable scenarios where a video contains either a single ($N=1$) or multiple ($N>1$) actor(s) associated with a single ($K=1$) or multiple ($K>1$) action(s).

\section{Experiments}
    \label{sec:exp}
    Our method is implemented based on PyTorch~\cite{paszke2019pytorch}. All models are trained on $4$ NVIDIA Quadro RTX 8000 GPUs. The code will be released upon acceptance.

\begin{table*}[t]
\footnotesize
    \centering
    \rowcolors{3}{white}{gray!15}
    \setlength\tabcolsep{8pt}
    \begin{tabular}{c|c|c|c|c|c}
        \toprule\rowcolor{LightBlue} 
         & \multicolumn{4}{c|}{\textbf{PE} / \textbf{NE} / \textbf{PNE} / \textbf{Belief}} & \textbf{hours}\\ \rowcolor{LightBlue} 
        \multirow{-2}{*}{\textbf{$m$}} & \textbf{Error}$\downarrow$ & \textbf{AUROC}$\uparrow$ & \textbf{AUPR}$\uparrow$ & \textbf{FPR at 95\% TPR}$\downarrow$ & \textbf{per epoch}\\
        \cmidrule(r){1-1} \cmidrule(lr){2-2} \cmidrule(lr){3-3} \cmidrule(lr){4-4} \cmidrule(l){5-5}\cmidrule(l){6-6}
        $0$ & 31.25 / 35.89 / 25.17 / 40.82 & 75.23 / 70.01 / 75.04 / 69.45 & 86.10 / 83.65 / 87.33 / 90.87 & 7.32 / 9.21 / 8.98 / 8.81 & 4\\
        $1$ & 11.21 / 41.20 / 11.16 / 33.33 & 86.42 / 61.26 / \textbf{85.72} / 73.34 & 94.23 / \textbf{96.47} / \textbf{99.43} / \textbf{97.07} & \textbf{4.32} / 5.01 / 5.46 / \textbf{9.09} & 5\\
        $2$ & 11.22 / \textbf{40.14} / 12.18 / 27.91 & \textbf{86.92} / \textbf{63.54} / 85.25 / 83.41 & 96.18 / 89.90 / 93.46 / 90.90 & 4.98 / 5.09 / 4.51 / 9.15 & 9\\
        $3$ & 11.01 / 45.05 / 11.17 / 27.86 & 86.81 / 58.23 / 85.25 / \textbf{85.18} & 98.48 / 95.65 / 99.41 / 90.44 & 4.58 / \textbf{5.00} / 3.98 / \textbf{9.09} & 13\\
        $4$ & 11.28 / 47.12 / 11.17 / 27.52 & 86.56 / 59.01 / 85.13 / 84.72 & 99.43 / 88.98 / 99.40 / 90.42 & 4.64 / 5.02 / 3.21 / \textbf{9.09} & 18\\
        $5$ & \textbf{10.22} / 44.32 / \textbf{11.15} / \textbf{27.49} & 86.86 / 58.87 / 85.30 / 84.66 & \textbf{99.52} / 96.35 / 99.41 / 90.41 & \textbf{4.32} / \textbf{5.00} / \textbf{3.09} / \textbf{9.09} & 23\\
        \bottomrule
    \end{tabular}
    \vspace{-3mm}
    \caption{Exploration of number of average primal-dual updating step $m$ on AVA~\cite{gu-2018-cvpr-ava}. }
    \vspace{-2mm}
    \label{tab:ablation-k_iter}
\end{table*}

\begin{table*}[t]
\footnotesize
    \centering  
  \rowcolors{3}{white}{gray!15}
    \setlength\tabcolsep{3pt}
    \begin{tabular}{l|c|c|c|c|c}
        \toprule \rowcolor{LightBlue} 
         & \multicolumn{4}{c|}{\textbf{PE} / \textbf{NE} / \textbf{PNE} / \textbf{Belief}} & \textbf{Closed Set}\\ \rowcolor{LightBlue} 
        \multirow{-2}{*}{\textbf{Method}} & \textbf{Error}$\downarrow$ & \textbf{AUROC}$\uparrow$ & \textbf{AUPR}$\uparrow$ & \textbf{FPR at 95\% TPR}$\downarrow$ & \textbf{mAP}$\uparrow$\\
        \cmidrule(r){1-1} \cmidrule(lr){2-2} \cmidrule(lr){3-3} \cmidrule(lr){4-4} \cmidrule(l){5-5} \cmidrule(l){6-6}        
         \rowcolor{LightGreen} \sysname{}, R-50 & 11.22 / 40.14 / 12.18 / 27.91 & 86.92 / 63.54 / 85.25 / 83.41 & 96.18 / 89.90 / 93.46 / 90.90 & 4.98 / 5.09 / 4.51 / 9.15 & 27.80\\
        \midrule
        w/o \sysnameacor{} & 50.12 / 50.04 / 12.45 / 39.84 & 51.23 / 51.12 / 84.88 / 82.41  & 86.11 / 86.36 / 88.13 / 92.35 & 15.23 / 15.56 / 34.44 / 9.57 & 25.12\\
        w/o \sysnamemenn{} & 35.12 / 35.31 / 35.12 / 34.78 & 57.10 / 57.59 / 57.88 / 57.01 & 85.71 / 85.65 / 85.66 / 85.66 & 12.03 / 13.00 / 14.98 / 15.32 & 28.81\\
        w/o \sysnameedc{} & 13.16 / 46.32 / 12.16 / 38.23 & 86.00 / 50.13 / 85.12 / 83.05 & 95.12 / 90.67 / 92.45 / 89.19 & 6.31 / 5.05 / 5.01 / 9.59 & 27.16\\
        \bottomrule
    \end{tabular}
    \vspace{-3mm}
    \caption{Exploration of different component in \sysname{} with $m=2$ on AVA~\cite{gu-2018-cvpr-ava}. Our complete system is highlighted in \textcolor{Green}{green}.}
    \vspace{-5mm}
    \label{tab:ablation-components}
\end{table*}

\textbf{Datasets.}
Two video datasets covering different cases are used in our experiment.
AVA~\cite{gu-2018-cvpr-ava} is a video dataset for spatio-temporal localizing atomic visual actions. It contains $430$ videos, each with $15$ minutes annotated in $1$-second intervals. Box annotations and their corresponding action labels are provided on key frames. We use version 2.2 of the AVA dataset by default. 
Charades~\cite{sigurdsson2016hollywood_Charades} contains $9,848$ videos that average $30$ seconds in length. This dataset includes $157$ multi-label, daily indoor activities. 

\textbf{Implementation details.} 
Similar to~\cite{pan2020actorcontextactor}, we employ COCO~\cite{lin-2014-eccv} pre-trained Faster R-CNN~\cite{ren-2015-faster} with a ResNeXt-101-FPN~\cite{lin-2017-CVPR} backbone to extract actor and object proposals on key frames. 
To extract context features, Kinetics~\cite{kay-2017-kinetics} pre-trained SlowFast networks~\cite{DBLP:conf/iccv/Feichtenhofer0M19} are used as the backbone of our method. For AVA~\cite{gu-2018-cvpr-ava}, the inputs are $64$-frame clips, where we sample $T=8$ frames with a temporal stride $\tau=8$ for the slow pathway, and $\zeta T$ frames, where $\zeta=4$, for the fast pathway. 
For Charades~\cite{sigurdsson2016hollywood_Charades}, the temporal sampling for the slow pathway is changed to $8\times4$, and the fast pathway takes as an input $32$ continuous frames. 
We train all models end-to-end using synchronous SGD with a batch size of $32$. Linear warm-up~\cite{goyal-2017-arXiv} is performed during the first several epochs. We used both ground-truth boxes and predicted human boxes from~\cite{wu-2019-CVPR} for training. We scale the shorter side of input frames to $256$ pixels for testing. We use detected human boxes with scores greater than $0.85$ for final action detection. 

\subsection{Open- and Closed-Set Settings}
Since the above datasets are used for traditional action recognition, we re-split them to adapt to our problem in this work.
For open-set settings, videos are evenly divided into three disjoint sets $\mathcal{Z}_1, \mathcal{Z}_2$, and $\mathcal{Z}_3$. We only include actions falling in $\mathcal{Z}_1\cup\mathcal{Z}_2$ for training and actions in $\mathcal{Z}_2\cup\mathcal{Z}_3$ for testing.
Thus $\mathcal{Z}_2$ and $\mathcal{Z}_3$ are a set of known actions and novel actions in inference, respectively.
In practice, each subset in AVA~\cite{gu-2018-cvpr-ava} contains $20$ actions, while actions in Charades~\cite{sigurdsson2016hollywood_Charades} are evenly divided into three subsets ($52/52/53$). 
An actor is considered as novelty (unknown) if it does not contain any action in the training data. 
To detect actors with novel actions, in the testing stage, we ensure that each actor contains ground-truth actions in either $\mathcal{Z}_2$ or $\mathcal{Z}_3$ exclusively \cite{DBLP:conf/nips/WangLBL21}. 
We hence assign each actor in testing videos with a binary novelty label $\{0,1\}$, where $0$ indicates an actor with all known actions in $\mathcal{Z}_2$ and correspondingly, $1$ indicates an actor with all unknown actions in $\mathcal{Z}_3$. 

Closed set action recognition refers to classifying actions into pre-defined categories. Following~\cite{pan2020actorcontextactor}, the closed set studies on AVA~\cite{gu-2018-cvpr-ava} use $235$ videos to train and test on $131$ videos with known actions. 
For Charades~\cite{sigurdsson2016hollywood_Charades}, following~\cite{DBLP:conf/cvpr/ZhangLM21}, we use the officially provided train-test split ($7,985/1,863$) to evaluate the network where all actions are known. 
As this work focuses on open set action recognition, closed set accuracy is for reference only.

\textbf{Evaluation metrics.}
Similar to~\cite{liang2017enhancing,liu2020energy}, we adopt the following four metrics to evaluate the performance on novel action detection, \ie, estimate if the action of an actor is novel or not. 
1) \textbf{Detection Error}~\cite{liang2017enhancing} measures the misclassification probability when True Positive Rate (TPR) is $95\%$. The definition of an error $P_e$ is given by $P_e=0.5\cdot(1-\text{TPR})+0.5\cdot\text{FPR}$, where FPR stands for False Positive Rate. 
2) \textbf{AUROC}~\cite{davis2006relationship} is the Area Under the Receiver Operating Characteristic curve, which depicts the relation between TPR and FPR. A perfect detector corresponds to an AUROC score of $1$. 
3) \textbf{AUPR}~\cite{Manning-1999-AUPR} is the Area under the Precision-Recall curve. The PR curve is a graph showing the precision and recall against each other. 
4) \textbf{FPR at 95\% TPR}~\cite{liang2017enhancing} can be interpreted as the probability that a novel example is misclassified as known when TPR is $95\%$.
Additionally, we report the Mean Average Precision (\textbf{mAP}) for $K$-class classification in closed set. 

\subsection{Ablation Study}
To further explore our method, we conduct a detailed ablation study on AVA~\cite{gu-2018-cvpr-ava}. In the following tables, evaluation metrics with ``$\uparrow$" indicate the larger the better, and ``$\downarrow$" indicate the smaller the better.

\textbf{Effectiveness of optimization algorithm.}
We investigate the effectiveness of applying an averaging scheme to the primal sequence in \Cref{alg:algor}. 
According to \Cref{tab:ablation-k_iter}, the larger the primal-dual updating step $m$, the better performance and the lower efficiency. 
If $m=0$, it indicates that we do not use the proposed optimization method.
The dual parameter $\lambda$ in \Cref{eq:primal-solution} is then viewed as a Lagrangian multiplier and set empirically. The results demonstrate the effectiveness of our algorithm in reducing static bias.
Considering the trade-off between performance and efficiency, we use $m=2$ in the following experiments.
\begin{figure*}[t!]
    \centering
    \includegraphics[width=\linewidth]{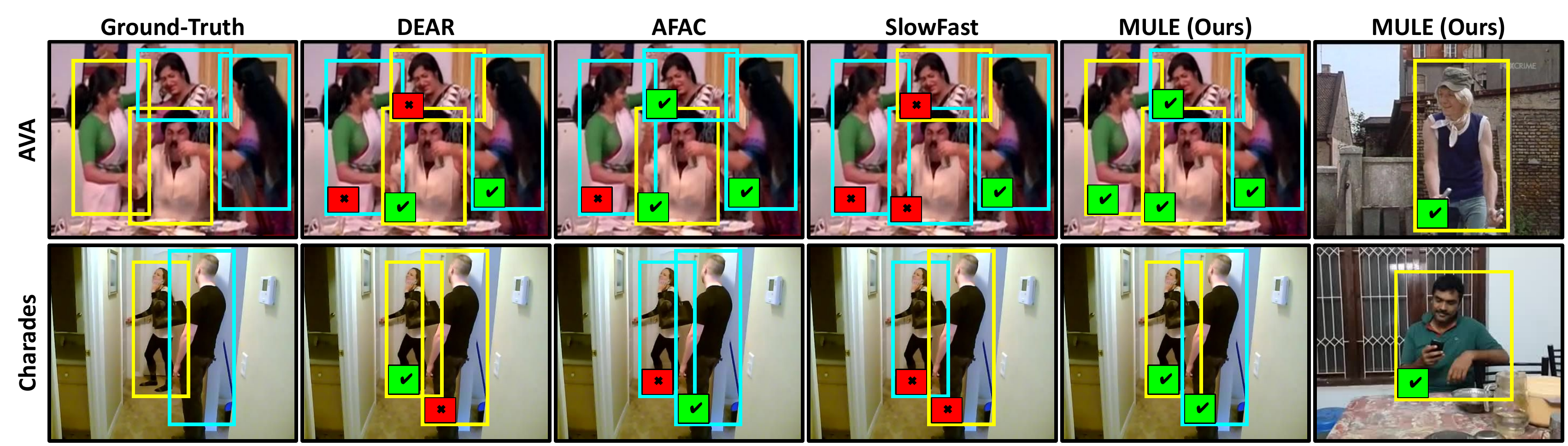}
    \vspace{-7mm}
    \caption{Visual comparison with our method and state-of-the-art on AVA~\cite{gu-2018-cvpr-ava} and Charades~\cite{sigurdsson2016hollywood_Charades}. \textcolor{Cyan}{Cyan} and (\textcolor{Goldenrod!80!black}{yellow}) boxes denote the predictions of actors with known and novel actions, respectively. \textcolor{Green}{\cmark ~marks} and \textcolor{Red}{\xmark~marks} indicate correct and false predictions, respectively.}
    \label{fig:visualization}
    \vspace{-2mm}
\end{figure*}

\textbf{Contribution of important components.}
In Table \ref{tab:ablation-components}, we discuss the contribution of each component of \sysname{} on both open- and closed-set settings as follows.
\begin{itemize}[leftmargin=*,topsep=0pt,itemsep=0ex,partopsep=0ex,parsep=0ex]
    \item \textbf{\sysnameacor{}.} Without \sysnameacor{} representation, we only rely on the video backbone (\ie, SlowFast~\cite{DBLP:conf/iccv/Feichtenhofer0M19}) and a single layer action classifier to predict actions in the video. The results show a significant performance drop in both closed set mAP ($25.12\%$ vs. $27.80\%$) and open set metrics, demonstrating the importance of \sysnameacor{} representation.
    \item \textbf{\sysnamemenn{}.} If we use Beta loss in \Cref{eq:beta-loss} in the network, the open set recognition performance is improved with a slight sacrifice of the closed set mAP ($28.81\%$ vs. $27.80\%$). This is because each class follows a binomial opinion whose conjugate prior is a Beta distribution to benefit open set recognition, rather than a Dirichlet distribution aiming to optimize multi-class classification in a closed-set setting. Although using the Beta loss does not outperform using cross-entropy in a closed-set setting, its performance is still competitive. 
    \item \textbf{\sysnameedc{}.} If we remove the evidence debiasing constraint in network training, our method will suffer from static bias, as explained in \Cref{sec:debias}. The results show that using our \sysnameedc{} module in training brings considerable improvement on all four novelty scores, showing the effectiveness to model static bias. Empirically, we set $\gamma=0.001$ in \Cref{eq:opt-problem}.
\end{itemize}
\textbf{Comparison with novelty estimation mechanisms.} 
As presented in \Cref{tab:ablation-k_iter,tab:ablation-components}, we compare four novelty estimation mechanisms in \Cref{sec:novelty_estimation}. It can be concluded that uncertainty-based scores perform better than belief based one. 
According to \Cref{eq:element-wise-sl}, a belief based score is calculated based on negative evidence. 
We speculate that it is more intuitive and accurate to estimate if an actor performs the same action rather than different actions. 
In summary, positive evidence is more reliable than negative evidence. In the following, we only report Positive Evidence (PE) scores to detect novel actions for clarity.
\begin{table}[t]
\scriptsize
\rowcolors{3}{white}{gray!15}
    \centering
    \setlength\tabcolsep{1.2pt}
    \begin{tabular}{l|c|c|c|c|c|c}
        \toprule  
        \rowcolor{LightBlue} \textbf{Methods} & \textbf{Pre-train} & \textbf{Error}$\downarrow$ & \textbf{AUROC}$\uparrow$ & \textbf{AUPR}$\uparrow$ & 
        \textbf{FPR at}$\downarrow$ & \textbf{Closed Set}\\
        \rowcolor{LightBlue} &&&&&\textbf{95\% TPR}& \textbf{mAP}$\uparrow$\\
        \cmidrule(r){1-1} \cmidrule(lr){2-2} \cmidrule(lr){3-3} \cmidrule(lr){4-4} \cmidrule(lr){5-5} \cmidrule(l){6-6} \cmidrule(l){7-7}
        Slowfast, R-101~\cite{DBLP:conf/iccv/Feichtenhofer0M19} & K600 & 60.12 & 50.15 & 70.15 & 20.17 & 29.00\\
        ACAR, R-50~\cite{pan2020actorcontextactor} & K400 & 35.16 & 52.89 & 79.15 & 14.16 & 28.84\\
        ACAR, R-101~\cite{pan2020actorcontextactor} & K700 & 32.26 & 55.18 & 82.15 & 10.16 & \textbf{33.30}\\
        AFAC, R-101~\cite{zhang2021multi}& K600 & 53.14 & 79.69 & 90.79 & 7.15 & 30.20\\ 
        AIA, R-101~\cite{tang2020asynchronous} & K700 & 35.14 & 54.17 & 78.49 & 10.15 & 32.30\\
        DEAR, R-50~\cite{bao-2021-ICCV}& K400 & 23.22 & 82.12 & 83.15 & 8.45 & 18.51\\
        \rowcolor{LightGreen}
        \midrule
        \sysname{} (Ours), R-50 & K400 & 11.22 & 86.92 & 96.18 & 4.98 & 27.80\\ \rowcolor{LightGreen}
        \sysname{} (Ours), R-101 & K700 & \textbf{10.12} & \textbf{88.75} & \textbf{98.18} & \textbf{4.17} & 29.87\\
        \bottomrule
    \end{tabular}
    \vspace{-3mm}
    \caption{Comparison with state-of-the-art on AVA~\cite{gu-2018-cvpr-ava}. Ours is highlighted in \textcolor{Green}{green}. Best value is in \textbf{bold}.}
    \vspace{-3mm}
    \label{tab:baseline_AVA}
\end{table}

\subsection{Results Analysis}
For a fair comparison with previous single-actor/single-action methods, we enhance them to work within the new multi-actor/multi-action paradigm.
Specifically, we calculate positive evidence by the function $\text{ReLU}(h(\mathbf{x};\boldsymbol{\theta}))+1$ (same as described in \Cref{sec:edl}), based on the last layer outputs $h(\mathbf{x};\boldsymbol{\theta})$. 
We apply single-actor based DEAR~\cite{bao-2021-ICCV} to a bounding volume around each detected actor for a multi-actor setting.

According to \Cref{tab:baseline_AVA}, our method is compared with existing methods with both ResNet-50 and ResNet-101 backbones on AVA~\cite{gu-2018-cvpr-ava}.
The baseline SlowFast~\cite{DBLP:conf/iccv/Feichtenhofer0M19} achieves a reasonable closed set mAP score but fails in open set recognition.
Other state-of-the-art methods including ACAR~\cite{pan2020actorcontextactor}, AFAC~\cite{zhang2021multi}, and AIA~\cite{tang2020asynchronous} obtain better accuracy in the closed set but are still unsatisfying to detect novel actions.
Compared with DEAR~\cite{bao-2021-ICCV} using ENNs, our method achieves much better performance in both closed set and open set metrics. This is because our network can handle the situation that each actor may contain more than one action. 
In addition, simple bounding volume strategy is not effective to extract global information in the multi-actor setting.
It indicates the effectiveness of the proposed \sysnamemenn{} for multi-label open set action recognition. 

Furthermore, we observe a similar trend on Charades~\cite{sigurdsson2016hollywood_Charades}.
From \Cref{tab:baseline_Charades}, our \sysname{} achieves comparable closed set mAP and much better open set performance compared with the state-of-the-art. 
It is worth mentioning that AFAC~\cite{zhang2021multi} and CSN~\cite{tran2019video} achieve the best closed set accuracy by using powerful CSN-152 backbone \cite{DBLP:conf/cvpr/ZhangLM21} and IG-65M pre-trained dataset. However, our method using ResNet backbone still obtains a considerable gain in terms of AUROC and FPR at 95\% TPR. 
It is a good fit for the situation that requires lower probability of misclassifying a novel instance as known.

As shown in \Cref{fig:visualization}, we provide several visual results on two datasets, where our method is compared with DEAR~\cite{bao-2021-ICCV}, AFAC~\cite{zhang2021multi} and SlowFast~\cite{DBLP:conf/iccv/Feichtenhofer0M19}. 
The examples show that other methods output several false predictions of known and novel actions in a multi-actor/multi-action setting.
In contrast, our method can handle novelty detection in different situations more accurately.
\begin{table}[t]
\scriptsize
\rowcolors{3}{white}{gray!15}
    \centering
    \setlength\tabcolsep{1.2pt}
    \begin{tabular}{l|c|c|c|c|c|c}
        \toprule  
        \rowcolor{LightBlue} \textbf{Methods} & \textbf{Pre-train} & \textbf{Error}$\downarrow$ & \textbf{AUROC}$\uparrow$ & \textbf{AUPR}$\uparrow$ & 
        \textbf{FPR at}$\downarrow$ & \textbf{Closed Set}\\
        \rowcolor{LightBlue} &&&&&\textbf{95\% TPR}& \textbf{mAP}$\uparrow$\\
        \cmidrule(r){1-1} \cmidrule(lr){2-2} \cmidrule(lr){3-3} \cmidrule(lr){4-4} \cmidrule(lr){5-5} \cmidrule(l){6-6} \cmidrule(l){7-7}
        Slowfast, R-101~\cite{DBLP:conf/iccv/Feichtenhofer0M19} & K600 & 25.15 & 79.12 & 79.15 & 50.15 & 45.20\\
        X3D-XL~\cite{feichtenhofer2020x3d} & K600 & 8.45 & 82.15 & 86.49 & 39.98 & 47.20\\ 
        AFAC, R-101~\cite{zhang2021multi} & K600 & 7.00 & 82.20 & 90.15 & 35.19 & 48.10 \\
        AFAC, CSN-152~\cite{zhang2021multi} & IG-65M & 6.18 & 80.12 & 90.79 & 30.15 & \textbf{50.30}\\
        CSN, CSN-152~\cite{tran2019video} & IG-65M & 6.89 & 80.15 & \textbf{92.61} & 35.16  & 46.40\\
        DEAR, R-50~\cite{bao-2021-ICCV}& K400 & 12.15 & 86.15 & 92.35 & 29.96 & 38.12\\
        \rowcolor{LightGreen}
        \midrule
        \sysname{} (Ours), R-50 & K400 & \textbf{6.15} & 85.49 & 90.78 & 25.98 & 45.33\\ \rowcolor{LightGreen}
        \sysname{} (Ours), R-101 & K700 & 6.23 & \textbf{88.49} & 91.15 & \textbf{25.15} & 47.21\\
        \bottomrule
    \end{tabular}
    \vspace{-3mm}
    \caption{Comparison with state-of-the-art on Charades~\cite{sigurdsson2016hollywood_Charades}. Ours is highlighted in \textcolor{Green}{green}. Best value is in \textbf{bold}.}
    \vspace{-3mm}
    \label{tab:baseline_Charades}
\end{table}

\section{Conclusion}
    \label{sec:conclusion}
    In this paper, we address the general problem of detecting the novelty of each actor's action(s) in video for the first time. 
To this end, we develop a new open set action recognition framework using Multi-Label Evidential Learning (\sysname{}).
Based on actor-context-object relation representation, the proposed Beta Evidential Neural Network can formulate Beta distributions for single/multi-actor, single/multi-action settings.
To optimize the network, we add the multi-label evidence debiasing constraint and propose the primal-dual average scheme update learning method with theoretical analysis.
Thus both direct and indirect dependencies of action predictions are reduced for more robust results.
The experiments show that Beta distribution more accurately classifies novel action(s) than the previous Dirichlet distribution.

{\small
\bibliographystyle{ieee_fullname}
\bibliography{egbib}

\begin{thebibliography}{10}\itemsep=-1pt

\bibitem{aggarwal2011human}
J.~K. Aggarwal and Michael~S. Ryoo.
\newblock Human activity analysis: {A} review.
\newblock {\em {ACM} Comput. Surv.}, 43(3):16:1--16:43, 2011.

\bibitem{NEURIPS2020_aab08546}
Alexander Amini, Wilko Schwarting, Ava Soleimany, and Daniela Rus.
\newblock Deep evidential regression.
\newblock In {\em NeurIPS}, 2020.

\bibitem{bahng-2020-ICML}
Hyojin Bahng, Sanghyuk Chun, Sangdoo Yun, Jaegul Choo, and Seong~Joon Oh.
\newblock Learning de-biased representations with biased representations.
\newblock In {\em ICML}, pages 528--539, 2020.

\bibitem{bao-2021-ICCV}
Wentao Bao, Qi Yu, and Yu Kong.
\newblock Evidential deep learning for open set action recognition.
\newblock In {\em ICCV}, pages 13329--13338, 2021.

\bibitem{Barocas-CLR-2016}
Solon Barocas and Andrew~D. Selbst.
\newblock Big data's disparate impact.
\newblock {\em California Law Review 104}, 2016.

\bibitem{bendale2016towards}
Abhijit Bendale and Terrance~E Boult.
\newblock Towards open set deep networks.
\newblock In {\em CVPR}, pages 1563--1572, 2016.

\bibitem{davis2006relationship}
Jesse Davis and Mark Goadrich.
\newblock The relationship between precision-recall and roc curves.
\newblock In {\em ICML}, pages 233--240, 2006.

\bibitem{fang2021learning}
Zhen Fang, Jie Lu, Anjin Liu, Feng Liu, and Guangquan Zhang.
\newblock Learning bounds for open-set learning.
\newblock In {\em ICML}, pages 3122--3132, 2021.

\bibitem{feichtenhofer2020x3d}
Christoph Feichtenhofer.
\newblock X3d: Expanding architectures for efficient video recognition.
\newblock In {\em CVPR}, pages 203--213, 2020.

\bibitem{DBLP:conf/iccv/Feichtenhofer0M19}
Christoph Feichtenhofer, Haoqi Fan, Jitendra Malik, and Kaiming He.
\newblock Slowfast networks for video recognition.
\newblock In {\em ICCV}, pages 6201--6210, 2019.

\bibitem{gal-2016-icml}
Yarin Gal and Zoubin Ghahramani.
\newblock Dropout as a bayesian approximation: Representing model uncertainty
  in deep learning.
\newblock In {\em ICML}, pages 1050--1059, 2016.

\bibitem{geng2020recent}
Chuanxing Geng, Sheng{-}Jun Huang, and Songcan Chen.
\newblock Recent advances in open set recognition: {A} survey.
\newblock {\em TPAMI}, 43(10):3614--3631, 2021.

\bibitem{goyal-2017-arXiv}
Priya Goyal, Piotr Doll{\'{a}}r, Ross~B. Girshick, Pieter Noordhuis, Lukasz
  Wesolowski, Aapo Kyrola, Andrew Tulloch, Yangqing Jia, and Kaiming He.
\newblock Accurate, large minibatch {SGD:} training imagenet in 1 hour.
\newblock {\em CoRR}, abs/1706.02677, 2017.

\bibitem{gu-2018-cvpr-ava}
Chunhui Gu, Chen Sun, David~A. Ross, Carl Vondrick, Caroline Pantofaru, Yeqing
  Li, Sudheendra Vijayanarasimhan, George Toderici, Susanna Ricco, Rahul
  Sukthankar, Cordelia Schmid, and Jitendra Malik.
\newblock {AVA:} {A} video dataset of spatio-temporally localized atomic visual
  actions.
\newblock In {\em CVPR}, pages 6047--6056, 2018.

\bibitem{he-2017-ICCV}
Kaiming He, Georgia Gkioxari, Piotr Doll{\'a}r, and Ross Girshick.
\newblock Mask r-cnn.
\newblock In {\em ICCV}, pages 2961--2969, 2017.

\bibitem{josang-2016-subjective}
Audun J{\o}sang.
\newblock Subjective logic: A formalism for reasoning under uncertainty.
\newblock {\em Springer Verlag}, 2016.

\bibitem{Audun2006BeliefCalculus}
Audun Jøsang.
\newblock Belief calculus.
\newblock {\em arXiv preprint arXiv:0606029}, 2006.

\bibitem{kay-2017-kinetics}
Will Kay, Jo{\~{a}}o Carreira, Karen Simonyan, Brian Zhang, Chloe Hillier,
  Sudheendra Vijayanarasimhan, Fabio Viola, Tim Green, Trevor Back, Paul
  Natsev, Mustafa Suleyman, and Andrew Zisserman.
\newblock The kinetics human action video dataset.
\newblock {\em CoRR}, abs/1705.06950, 2017.

\bibitem{kendall2017uncertainties}
Alex Kendall and Yarin Gal.
\newblock What uncertainties do we need in bayesian deep learning for computer
  vision?
\newblock In Isabelle Guyon, Ulrike von Luxburg, Samy Bengio, Hanna~M. Wallach,
  Rob Fergus, S.~V.~N. Vishwanathan, and Roman Garnett, editors, {\em NeurIPS},
  pages 5574--5584, 2017.

\bibitem{li-2018-eccv}
Yingwei Li, Yi Li, and Nuno Vasconcelos.
\newblock Resound: Towards action recognition without representation bias.
\newblock In {\em ECCV}, pages 513--528, 2018.

\bibitem{liang2017enhancing}
Shiyu Liang, Yixuan Li, and R. Srikant.
\newblock Enhancing the reliability of out-of-distribution image detection in
  neural networks.
\newblock In {\em ICLR}, 2018.

\bibitem{lin-2017-CVPR}
Tsung-Yi Lin, Piotr Doll{\'a}r, Ross Girshick, Kaiming He, Bharath Hariharan,
  and Serge Belongie.
\newblock Feature pyramid networks for object detection.
\newblock In {\em CVPR}, pages 2117--2125, 2017.

\bibitem{lin-2014-eccv}
Tsung-Yi Lin, Michael Maire, Serge Belongie, James Hays, Pietro Perona, Deva
  Ramanan, Piotr Doll{\'a}r, and C~Lawrence Zitnick.
\newblock Microsoft coco: Common objects in context.
\newblock In {\em ECCV}, pages 740--755, 2014.

\bibitem{liu2020energy}
Weitang Liu, Xiaoyun Wang, John~D. Owens, and Yixuan Li.
\newblock Energy-based out-of-distribution detection.
\newblock In {\em NeurIPS}, 2020.

\bibitem{liu2017sphereface}
Weiyang Liu, Yandong Wen, Zhiding Yu, Ming Li, Bhiksha Raj, and Le Song.
\newblock Sphereface: Deep hypersphere embedding for face recognition.
\newblock In {\em CVPR}, pages 212--220, 2017.

\bibitem{Manning-1999-AUPR}
Chris Manning and Hinrich Schütze.
\newblock Foundations of statistical natural language processing.
\newblock {\em MIT Press}, 1999.

\bibitem{averagedDS-SJO-2009}
Angelia Nedic and Asuman~E. Ozdaglar.
\newblock Approximate primal solutions and rate analysis for dual subgradient
  methods.
\newblock {\em {SIAM} J. Optim.}, 19(4):1757--1780, 2009.

\bibitem{pan2020actorcontextactor}
Junting Pan, Siyu Chen, Mike~Zheng Shou, Yu Liu, Jing Shao, and Hongsheng Li.
\newblock Actor-context-actor relation network for spatio-temporal action
  localization.
\newblock In {\em CVPR}, pages 464--474, 2021.

\bibitem{paszke2019pytorch}
Adam Paszke, Sam Gross, Francisco Massa, Adam Lerer, James Bradbury, Gregory
  Chanan, Trevor Killeen, Zeming Lin, Natalia Gimelshein, Luca Antiga, Alban
  Desmaison, Andreas K{\"{o}}pf, Edward~Z. Yang, Zachary DeVito, Martin Raison,
  Alykhan Tejani, Sasank Chilamkurthy, Benoit Steiner, Lu Fang, Junjie Bai, and
  Soumith Chintala.
\newblock Pytorch: An imperative style, high-performance deep learning library.
\newblock In {\em NeurIPS}, pages 8024--8035, 2019.

\bibitem{ren-2015-faster}
Shaoqing Ren, Kaiming He, Ross Girshick, and Jian Sun.
\newblock Faster r-cnn: Towards real-time object detection with region proposal
  networks.
\newblock In {\em NeurIPS}, volume~28, 2015.

\bibitem{DBLP:conf/bmvc/RoitbergAS18}
Alina Roitberg, Ziad Al{-}Halah, and Rainer Stiefelhagen.
\newblock Informed democracy: Voting-based novelty detection for action
  recognition.
\newblock In {\em BMVC}, page~52, 2018.

\bibitem{roitberg2020open}
Alina Roitberg, Chaoxiang Ma, Monica Haurilet, and Rainer Stiefelhagen.
\newblock Open set driver activity recognition.
\newblock In {\em IV}, pages 1048--1053, 2020.

\bibitem{DBLP:journals/pami/RuddJSB18}
Ethan~M. Rudd, Lalit~P. Jain, Walter~J. Scheirer, and Terrance~E. Boult.
\newblock The extreme value machine.
\newblock {\em TPAMI}, 40(3):762--768, 2018.

\bibitem{sensoy-2018-nips}
Murat Sensoy, Lance Kaplan, and Melih Kandemir.
\newblock Evidential deep learning to quantify classification uncertainty.
\newblock In {\em NeurIPS}, volume~31, 2018.

\bibitem{sigurdsson2016hollywood_Charades}
Gunnar~A. Sigurdsson, G{\"{u}}l Varol, Xiaolong Wang, Ali Farhadi, Ivan Laptev,
  and Abhinav Gupta.
\newblock Hollywood in homes: Crowdsourcing data collection for activity
  understanding.
\newblock In {\em ECCV}, volume 9905, pages 510--526, 2016.

\bibitem{sozykin2018multi}
Konstantin Sozykin, Stanislav Protasov, Adil Khan, Rasheed Hussain, and
  Jooyoung Lee.
\newblock Multi-label class-imbalanced action recognition in hockey videos via
  3d convolutional neural networks.
\newblock In {\em SNPD}, pages 146--151, 2018.

\bibitem{sun-2018-ECCV}
Chen Sun, Abhinav Shrivastava, Carl Vondrick, Kevin Murphy, Rahul Sukthankar,
  and Cordelia Schmid.
\newblock Actor-centric relation network.
\newblock In {\em ECCV}, pages 318--334, 2018.

\bibitem{tang2020asynchronous}
Jiajun Tang, Jin Xia, Xinzhi Mu, Bo Pang, and Cewu Lu.
\newblock Asynchronous interaction aggregation for action detection.
\newblock In {\em ECCV}, pages 71--87, 2020.

\bibitem{tran2019video}
Du Tran, Heng Wang, Lorenzo Torresani, and Matt Feiszli.
\newblock Video classification with channel-separated convolutional networks.
\newblock In {\em ICCV}, pages 5552--5561, 2019.

\bibitem{DBLP:conf/nips/WangLBL21}
Haoran Wang, Weitang Liu, Alex Bocchieri, and Yixuan Li.
\newblock Can multi-label classification networks know what they don't know?
\newblock In {\em NeurIPS}, pages 29074--29087, 2021.

\bibitem{wu-2019-CVPR}
Chao-Yuan Wu, Christoph Feichtenhofer, Haoqi Fan, Kaiming He, Philipp
  Krahenbuhl, and Ross Girshick.
\newblock Long-term feature banks for detailed video understanding.
\newblock In {\em CVPR}, pages 284--293, 2019.

\bibitem{yager-2008-classic}
Ronald~R. Yager and Liping Liu, editors.
\newblock {\em Classic Works of the Dempster-Shafer Theory of Belief
  Functions}, volume 219 of {\em Studies in Fuzziness and Soft Computing}.
\newblock Springer, 2008.

\bibitem{Zhang-AAAI-2018}
Junzhe Zhang and Elias Bareinboim.
\newblock Fairness in decision-making - the causal explanation formula.
\newblock In {\em AAAI}, pages 2037--2045, 2018.

\bibitem{zhang2021multi}
Yanyi Zhang, Xinyu Li, and Ivan Marsic.
\newblock Multi-label activity recognition using activity-specific features and
  activity correlations.
\newblock In {\em CVPR}, pages 14625--14635, 2021.

\bibitem{DBLP:conf/cvpr/ZhangLM21}
Yanyi Zhang, Xinyu Li, and Ivan Marsic.
\newblock Multi-label activity recognition using activity-specific features and
  activity correlations.
\newblock In {\em CVPR}, pages 14625--14635, 2021.

\bibitem{Zhao-ICDM-2019}
Chen Zhao and Feng Chen.
\newblock Rank-based multi-task learning for fair regression.
\newblock In {\em ICDM}, pages 916--925, 2019.

\bibitem{zhao-KDD-2021}
Chen Zhao, Feng Chen, and Bhavani Thuraisingham.
\newblock Fairness-aware online meta-learning.
\newblock In {\em KDD}, pages 2294--2304, 2021.

\bibitem{zhao2020fair}
Chen Zhao, Changbin Li, Jincheng Li, and Feng Chen.
\newblock Fair meta-learning for few-shot classification.
\newblock In {\em ICKG}, pages 275--282, 2020.

\end{thebibliography}
}

\clearpage
\appendix
\section{Notations}
\label{sec:notations}
Vectors are denoted by lower case bold face letters, \textit{e.g.,} positive and negative evidence $\boldsymbol{\alpha,\beta}\in\mathbb{R}^k$. Vectors with subscripts of indices, such as $\alpha_i, \beta_i$ indicate the $i$-th entry in $\boldsymbol{\alpha,\beta}$. The Euclidean $\ell_2$-norm is denoted as $\lVert\cdot\rVert$. 
Scalars are denoted by lowercase italic letters, \textit{e.g.,} $\eta>0$. Matrices are denoted by capital italic letters, \textit{e.g.,} $X\in\mathbb{R}^{C\times H\times W}$. Proposition
$[u]_+$ denotes the  projection of $[u]$ on the nonnegative orthant.
\begin{table}[h]
    \centering
    \rowcolors{2}{white}{gray!15}
    \begin{tabular}{p{0.18\linewidth} p{0.75\linewidth}}
        \toprule\rowcolor{LightBlue} 
        \textbf{Notations} & \textbf{Descriptions} \\
        \midrule
        $N$ & Number of detected persons\\
        $M$ & Number of detected objects\\
        $\boldsymbol{\theta}$ & Parameters of the backbone to calculate \sysnameacor\\
        $\mathbf{x}$ & Input video\\
        $X$ & Context feature\\
        $C, H, W$ & Channel, height and width\\
        $A, O$ & Actor (person) and object feature\\
        $F$ & Actor-context-object relation feature\\
        $H$ & Actor relational feature\\
        $\omega$ & Subjective opinion\\
        $b,d$ & Belief and disbelief masses\\
        $u$ & Uncertainty mass\\
        $W$ & Non-informative prior weight\\
        $a$ & Base rate\\
        $K$ & Total number of classes\\
        $\alpha, \beta$ & Positive and negative evidence\\
        $p$ & Class probability\\
        $\gamma$ & The independence criterion relaxation of \sysnameedc{}\\
        $m$  & primal-dual updating step\\
        \bottomrule
    \end{tabular}
    \caption{Important notations and corresponding descriptions.}
    \label{tab:notations}
\vspace{-5mm}
\end{table}

\begin{table*}[t]
\scriptsize
    \centering
    \rowcolors{3}{white}{gray!15}
    \begin{tabular}{c|l|c|c|c|c}
        
        \toprule\rowcolor{LightBlue} 
        \textbf{m} & & \multicolumn{4}{c}{\textbf{PE} / \textbf{NE} / \textbf{PNE} / \textbf{Belief}}\\ \rowcolor{LightBlue} 
        \textbf{iterations} & \multirow{-2}{*}{\textbf{Method}} & \textbf{Error} $\downarrow$ & \textbf{AUROC} $\uparrow$ & \textbf{AUPR} $\uparrow$ & \textbf{FPR at 95\% TPR} $\downarrow$ \\
        \cmidrule(r){1-1} \cmidrule(lr){2-2} \cmidrule(lr){3-3} \cmidrule(lr){4-4} \cmidrule(l){5-5} \cmidrule(l){6-6}

        \rowcolor{LightGreen} &\sysname{} & 11.22 / 40.14 / 12.18 / 27.91 & 86.92 / 63.54 / 85.25 / 83.41 & 96.18 / 89.90 / 93.46 / 90.90 & 4.98 / 5.09 / 4.51 / 9.15\\
        \cmidrule{2-6}
        &w/o \sysnameacor{} & 50.12 / 50.04 / 12.45 / 39.84 & 51.23 / 51.12 / 84.88 / 82.41  & 86.11 / 86.36 / 88.13 / 92.35 & 15.23 / 15.56 / 34.44 / 9.57\\
        &w/o \sysnamemenn{} & 35.12 / 35.31 / 35.12 / 34.78 & 57.10 / 57.59 / 57.88 / 57.01 & 85.71 / 85.65 / 85.66 / 85.66 & 12.03 / 13.00 / 14.98 / 15.32\\
        \multirow{-5}{*}{\rotatebox[origin=c]{90}{$\mathbf{m=2}$}}  &w/o \sysnameedc{} & 13.16 / 46.32 / 12.16 / 38.23 & 86.00 / 50.13 / 85.12 / 83.05 & 95.12 / 90.67 / 92.45 / 89.19 & 6.31 / 5.05 / 5.01 / 9.59\\
        \bottomrule
        
        \rowcolor{LightGreen} &  \sysname{} & 11.01 / 45.05 / 11.17 / 27.86 & 86.81 / 58.23 / 85.25 / 85.18 & 98.48 / 95.65 / 99.41 / 90.44 & 4.58 / 5.00 / 3.98 / 9.09 \\
        \cmidrule{2-6}
        & w/o \sysnameacor{} & 48.34 / 50.32 / 17.77 / 47.79 & 52.78 / 50.15 / 83.34 / 79.26 & 86.39 / 83.39 / 93.93 / 88.93 & 13.35 / 15.01 / 25.93 / 10.00\\
        & w/o \sysnamemenn{} & 32.64 / 45.12 / 50.14 / 21.71 & 61.41 / 55.23 / 80.12 / 83.23 & 83.67 / 83.39 / 84.03 / 88.66 & 12.21 / 15.23 / 12.24 / 22.22 \\
        \multirow{-5}{*}{\rotatebox[origin=c]{90}{$\mathbf{m=3}$}} & w/o \sysnameedc{} & 10.12 / 48.12 / 11.86 / 38.43 & 80.12 / 52.01 / 84.38 / 73.94 & 93.34 / 93.41 / 98.92 / 89.33 & 5.15 / 6.03 / 7.41 / 10.70 \\
    
        \bottomrule
        
        \rowcolor{LightGreen} & \sysname{} & 11.28 / 47.12 / 11.17 / 27.52 & 86.56 / 59.01 / 85.13 / 84.72 & 99.43 / 88.98 / 99.40 / 90.42 & 4.64 / 5.02 / 3.21 / 9.09 \\
        \cmidrule{2-6}
        & w/o \sysnameacor{} & 50.21 / 51.71 / 17.64 / 49.35 & 52.35 / 50.01 / 84.17 / 80.33 & 83.36 / 83.49 / 94.97 / 88.21 & 11.02 / 12.65 / 33.33 / 10.23\\
        & w/o \sysnamemenn{} & 42.04 / 50.15 / 48.14 / 29.21 & 83.01 / 51.36 / 89.33 / 59.63 & 87.21 / 93.39 / 84.29 / 86.48 & 10.01 / 12.12 / 14.00 / 9.70\\
        \multirow{-5}{*}{\rotatebox[origin=c]{90}{$\mathbf{m=4}$}} & w/o \sysnameedc{} & 13.35 / 48.01 / 10.86 / 34.28 & 82.51 / 50.21 / 85.51 / 78.26 & 93.34 / 93.93 / 98.94 / 89.17 & 5.69 / 6.16 / 3.73 / 9.62\\
        
        \bottomrule
        
        \rowcolor{LightGreen} & \sysname{} & 10.22 / 44.32 / 11.15 / 27.49 & 86.86 / 58.87 / 85.30 / 84.66 & 99.52 / 96.35 / 99.41 / 90.41 & 4.32 / 5.00 / 3.09 / 9.09 \\
        \cmidrule{2-6}
        & w/o \sysnameacor{} & 49.61 / 50.00 / 16.61 / 48.82 & 51.01 / 51.97 / 84.78 / 81.29 & 82.16 / 83.33 / 98.05 / 88.09 & 10.23 / 12.68 / 37.04 / 9.16\\
        & w/o \sysnamemenn{} & 41.34 / 49.67 / 49.21 / 10.75 & 84.83 / 52.41 / 82.91 / 90.44 & 86.13 / 93.40 / 87.27 / 90.29 & 9.97 / 11.98 / 12.68 / 13.01\\
        \multirow{-5}{*}{\rotatebox[origin=c]{90}{$\mathbf{m=5}$}}& w/o \sysnameedc{} & 11.09 / 48.01 / 10.86 / 32.14 & 79.34 / 50.98 / 85.12 / 75.23 & 93.33 / 94.02 / 98.94 / 91.14 & 5.39 / 5.79 / 3.70 / 7.11\\
         
        \bottomrule
    \end{tabular}
    \caption{Exploration of different component in \sysname{} with $m=2,3,4,5$.}
    \label{tab:ablation-compinent_m345}
\end{table*}

\begin{table*}[t]
\scriptsize
\rowcolors{3}{white}{gray!15}
    \centering
    \setlength\tabcolsep{3.5pt}
    \begin{tabular}{l|c|c|c|c|c|c}
        \toprule  \rowcolor{LightBlue} 
        &  & \multicolumn{4}{c}{\textbf{PE} / \textbf{NE} / \textbf{PNE} / \textbf{Belief}} & \textbf{Closed Set}\\
         \rowcolor{LightBlue} \multirow{-2}{*}{\textbf{Backbone}} & \multirow{-2}{*}{\textbf{Pre-train}}  & \textbf{Error} $\downarrow$ & \textbf{AUROC} $\uparrow$ & \textbf{AUPR} $\uparrow$ & \textbf{FPR at 95\% TPR} $\downarrow$ & \textbf{mAP} $\uparrow$ \\
        \cmidrule(r){1-1} \cmidrule(lr){2-2} \cmidrule(lr){3-3} \cmidrule(lr){4-4} \cmidrule(lr){5-5} \cmidrule(l){6-6} \cmidrule(l){7-7}
        ACAR, R-50~\cite{pan2020actorcontextactor} & Kinetics-400 & 35.16 / 35.56 / 23.12 / 30.15 & 52.89 / 58.48 / 60.16 / 57.17 & 79.15 / 80.17 / 83.48 / 81.48 & 14.16 / 15.49 / 20.19 / 11.15 & 28.84\\
        ACAR, R-101~\cite{pan2020actorcontextactor} & Kinetics-700 & 32.26 / \textbf{34.12} / 20.54 / 30.16 & 55.18 / 58.98 / 63.18 / 60.18 & 82.15 / 81.48 / 85.48 / 90.01 & 10.16 / 15.01 / 19.48 / 10.00 & \textbf{33.30}\\
        AIA, R-101~\cite{tang2020asynchronous} & Kinetics-700 & 35.14 / 34.56 / 25.01 / 32.14 & 54.17 / 59.48 / 59.78 / 59.69 & 78.49 / 80.17 / 86.15 / \textbf{91.48} & 10.15 / 13.98 / 18.48 / 9.68 & 32.30\\
        Slowfast, R-101~\cite{DBLP:conf/iccv/Feichtenhofer0M19} & Kinetics-600 & 60.12 / 59.12 / 23.15 / 40.15 & 50.15 / 52.23 / 69.14 / 68.15 & 70.15 / 75.15 / 69.68 / 68.21 & 20.17 / 19.56 / 23.12 / 26.15 & 29.00\\
        DEAR, R-50~\cite{bao-2021-ICCV}& Kinetics-400 & 23.22 / 42.15 / 23.15 / 30.19 & 82.12 / 60.12 / 80.48 / 83.59 & 83.15 / 88.14 / 90.15 / 85.49 & 8.45 / 8.48 / 6.30 / 13.15 & 18.51\\
        AFAC, R-101~\cite{zhang2021multi}& Kinetics-600 & 53.14 / 49.41 / 40.15 / 59.17 & 79.69 / \textbf{65.89} / 80.48 / 79.79 & 90.79 / 89.18 / 88.18 / 85.15 & 7.15 / 8.79 / \textbf{4.01} / \textbf{8.98} & 30.20\\ \rowcolor{LightGreen}
        \midrule
        \textbf{Ours}, R-50 & Kinetics-400 & 11.22 / 40.14 / 12.18 / 27.91 & 86.92 / 63.54 / 85.25 / 83.41 & 96.18 / 89.90 / 93.46 / 90.90 & 4.98 / \textbf{5.09} / 4.51 / 9.15 & 27.80\\ \rowcolor{LightGreen}
        \textbf{Ours}, R-101 & Kinetics-700 & \textbf{10.12} / 40.15 / \textbf{10.56} / \textbf{25.02} & \textbf{88.75} / 65.36 / \textbf{89.48} / \textbf{84.26} & \textbf{98.18} / \textbf{89.95} / \textbf{94.74} / 90.49 & \textbf{4.17} / 5.28 / 4.25 / 10.01 & 29.87\\
        \bottomrule
    \end{tabular}
    \caption{Comparison with state-of-the-art on AVA~\cite{gu-2018-cvpr-ava}. Ours is highlighted in \textcolor{Green}{green}. Best value is in \textbf{bold}.}
    \label{tab:baseline_AVA_full}
\end{table*}

\begin{table*}[t]
\scriptsize
\rowcolors{3}{white}{gray!15}
    \centering
    \setlength\tabcolsep{3.3pt}
    \begin{tabular}{l|c|c|c|c|c|c}
        \toprule  \rowcolor{LightBlue} 
        &  & \multicolumn{4}{c}{\textbf{PE} / \textbf{NE} / \textbf{PNE} / \textbf{Belief}} & \textbf{Closed Set}\\
         \rowcolor{LightBlue} \multirow{-2}{*}{\textbf{Backbone}} & \multirow{-2}{*}{\textbf{Pre-train}}  & \textbf{Error} $\downarrow$ & \textbf{AUROC} $\uparrow$ & \textbf{AUPR} $\uparrow$ & \textbf{FPR at 95\% TPR} $\downarrow$ & \textbf{mAP} $\uparrow$\\
        \cmidrule(r){1-1} \cmidrule(lr){2-2} \cmidrule(lr){3-3} \cmidrule(lr){4-4} \cmidrule(lr){5-5} \cmidrule(l){6-6} \cmidrule(l){7-7}
        AFAC, R-101~\cite{zhang2021multi} & Kinetics-600 & 7.00 / 26.62 / 8.88 / 19.17 & 82.20 / 69.62 / 82.15 / 87.00 & 90.15 / 85.12 / 90.15 / 89.15 & 35.19 / 29.55 / \textbf{23.15} / 31.15 & 48.10 \\
        AFAC, CSN-152~\cite{zhang2021multi} & IG-65M & 6.18 / 19.78 / 8.49 / \textbf{15.59} & 80.12 / 61.02 / \textbf{90.48} / 84.88 & 90.79 / 86.36 / 91.15 / 85.67 & 30.15 / 33.15 / 28.42 / 30.01 & \textbf{50.30}\\
        CSN, CSN-152~\cite{tran2019video} & IG-65M & 6.89 / 25.18 / 9.02 / 18.90 & 80.15 / 70.15 / 90.15 / 87.01 & \textbf{92.61} / 90.17 / 90.99 / 85.15 & 35.16 / 30.15 / 35.98 / 26.69 & 46.40\\
        Slowfast, R-101~\cite{DBLP:conf/iccv/Feichtenhofer0M19} & Kinetics-600 & 25.15 / 30.15 / 46.12 / 23.00 & 79.12 / 75.12 / 75.36 / 78.89 & 79.15 / 80.15 / 80.46 / 80.08 & 50.15 / 56.15 / 46.12 / 29.16 & 45.20\\
        DEAR, R-50~\cite{bao-2021-ICCV}& Kinetics-400 & 12.15 / 25.59 / 11.11 / 16.98 & 86.15 / \textbf{79.15} / 82.46 / 82.55 & 92.35 / 90.33 / 89.26 / 90.19 & 29.96 / 26.16 / 23.99 / 28.98 & 38.12\\
        X3D-XL~\cite{feichtenhofer2020x3d} & Kinetics-600 & 8.45 / 25.85 / 10.15 / 15.51 & 82.15 / 66.64 / 85.00 / 82.15 & 86.49 / 80.16 / 90.15 / 88.15 & 39.98 / 36.14 / 26.15 / 30.55 & 47.20\\ \rowcolor{LightGreen}
        \midrule
        \textbf{Ours}, R-50 & Kinetics-400 & \textbf{6.15} / 23.15 / 7.89 / 17.11 & 85.49 / 65.17 / 88.79 / \textbf{88.49} & 90.78 / \textbf{92.78} / 94.15 / \textbf{95.03} & 25.98 / 25.29 / 25.49 / \textbf{23.64} & 45.33\\ \rowcolor{LightGreen}
        \textbf{Ours}, R-101 & Kinetics-700 & 6.23 / \textbf{22.15} / \textbf{7.02} / 20.15 & \textbf{88.49} / 65.01 / 89.41 / 82.16 & 91.15 / 90.78 / \textbf{95.42} / 90.15 & \textbf{25.15} / \textbf{21.16} / 26.46 / 28.98 & 47.21\\
        \bottomrule
    \end{tabular}
    \caption{Comparison with state-of-the-art on Charades~\cite{sigurdsson2016hollywood_Charades}. Ours is highlighted in \textcolor{Green}{green}. Best value is in \textbf{bold}.}
    \label{tab:baseline_Charades_full}
\end{table*} 
   
\section{Proof Sketch of \cref{prop1}}
\label{sec:proof of prop1}
 
\noindent \textbf{Proposition 1.} \textit{(Convergence of Averaged Primal Sequence) Under \cref{assump1}, when the convex set $\Theta$ is compact, let the approximate primal sequence $\{\Tilde{\boldsymbol{\theta}}^{(m)}\}^{\infty}_{m=1}$ be the running averages of the primal iterates given in \Cref{eq:primal-avg}. Then $\{\Tilde{\boldsymbol{\theta}}^{(m)}\}^{\infty}_{m=1}$ can converge to its limit $\Tilde{\boldsymbol{\theta}}^*$.}

To prove the convergence in \cref{prop1}, we first prove the below \cref{lemma:cauchy} that
\begin{lemma}
\label{lemma:cauchy}
The approximate primal sequence $\{\Tilde{\boldsymbol{\theta}}^{(m)}\}^{\infty}_{m=1}$ given in \Cref{eq:primal-avg} is a Cauchy sequence. That is $\forall \epsilon>0$, there is a $Q\in\mathbb{N}$ such that $||\Tilde{\boldsymbol{\theta}}^{(m')}-\Tilde{\boldsymbol{\theta}}^{(m)}||\leq\epsilon$, $\forall m',m\geq Q$.
\end{lemma}

\begin{proof}
Given \Cref{eq:primal-avg}, we derive
\begin{align*}
    \Tilde{\boldsymbol{\theta}}^{(m+1)} 
    &= \frac{1}{m+1} \sum_{i=1}^{m} \boldsymbol{\theta}^{(i)} \\
    &= \frac{1}{m+1} \big(\boldsymbol{\theta}^{(m)}+\sum_{i=1}^{m-1} \boldsymbol{\theta}^{(i)} \big) \\
    &= \frac{1}{m+1} \boldsymbol{\theta}^{(m)} + \frac{m}{m+1}\cdot\frac{1}{m}\sum_{i=1}^{m-1} \boldsymbol{\theta}^{(i)}\\
    &= \frac{1}{m+1} \boldsymbol{\theta}^{(m)} + \frac{m}{m+1}\Tilde{\boldsymbol{\theta}}^{(m)} \\
    &= \frac{1}{m+1} \boldsymbol{\theta}^{(m)} + \Tilde{\boldsymbol{\theta}}^{(m)} -\frac{1}{m+1} \Tilde{\boldsymbol{\theta}}^{(m)}\\
    \Tilde{\boldsymbol{\theta}}^{(m+1)}-\Tilde{\boldsymbol{\theta}}^{(m)} &= \frac{1}{m+1}(\boldsymbol{\theta}^{(m)}-\Tilde{\boldsymbol{\theta}}^{(m)}).
\end{align*}
Under \cref{assump1}, $\Theta$ is a compact convex set and $\Tilde{\boldsymbol{\theta}}^{(m)},\boldsymbol{\theta}^{(m)} \in\Theta$. Let $m'>m$ and $||\Tilde{\boldsymbol{\theta}}^{(m)}||,||\boldsymbol{\theta}^{(m)}||\leq G$, we have
\begin{align*}
    &\quad||\Tilde{\boldsymbol{\theta}}^{(m')}-\Tilde{\boldsymbol{\theta}}^{(m)}||\\
    &= ||\Tilde{\boldsymbol{\theta}}^{(m')}-\Tilde{\boldsymbol{\theta}}^{(m'-1)}+\cdots+\Tilde{\boldsymbol{\theta}}^{(m+1)}-\Tilde{\boldsymbol{\theta}}^{(m)}||\\
    &=||\frac{\boldsymbol{\theta}^{(m'-1)}-\Tilde{\boldsymbol{\theta}}^{(m'-1)}}{m'}+\cdots+\frac{\boldsymbol{\theta}^{(m)}-\Tilde{\boldsymbol{\theta}}^{(m)}}{m+1}||\\
    &\leq \frac{||\boldsymbol{\theta}^{(m'-1)}||+||\Tilde{\boldsymbol{\theta}}^{(m'-1)}||}{m'}+\cdots+\frac{||\boldsymbol{\theta}^{(m)}||+||\Tilde{\boldsymbol{\theta}}^{(m)}||}{m+1}\\
    &\leq \frac{2G(m'-m)}{m+1}.
\end{align*}
Therefore, for any arbitrary $\epsilon>0$, let $\frac{2G(m'-m)}{m+1}<\epsilon$, and we have $||\Tilde{\boldsymbol{\theta}}^{(m')}-\Tilde{\boldsymbol{\theta}}^{(m)}||\leq \epsilon$. Therefore we conclude that $\{\Tilde{\boldsymbol{\theta}}^{(m)}\}^{\infty}_{m=1}$ is a Cauchy sequence.
\end{proof}

Next we prove the proposed \cref{prop1}.

\begin{proof}
As stated in \cref{lemma:cauchy} that $\{\Tilde{\boldsymbol{\theta}}^{(m)}\}$ in \Cref{eq:primal-avg} is a Cauchy sequence, it hence is bounded and there exists a subsequence $b_n$ converging to its limit $L$. For any $\epsilon>0$, there exists $n,q\geq Q$ satisfying $||\Tilde{\boldsymbol{\theta}}^{(n)} - \Tilde{\boldsymbol{\theta}}^{(q)}|| < \frac{\epsilon}{2}$. Thus, there is a $b_m=\Tilde{\boldsymbol{\theta}}^{(q_m)}$, such that $q_m\geq Q$ and $||b_{q_m}-L||<\frac{\epsilon}{2}$.
\begin{align*}
    ||\Tilde{\boldsymbol{\theta}}^{(n)} - L||
    &= ||\Tilde{\boldsymbol{\theta}}^{(n)} - b_m + b_m - L|| \\
    &\leq ||\Tilde{\boldsymbol{\theta}}^{(n)} - b_m|| + ||b_m - L|| \\
    &< ||\Tilde{\boldsymbol{\theta}}^{(n)} - \Tilde{\boldsymbol{\theta}}^{(q)}|| + \frac{\epsilon}{2} \\
    &< \epsilon.
\end{align*}
Since $\epsilon$ is arbitrarily small, we prove that the sequence $\{\Tilde{\boldsymbol{\theta}}^{(m)}\}^{\infty}_{m=1}$ converges to its limit $L = \Tilde{\boldsymbol{\theta}}^*$ asymptotically.
\end{proof}

\section{Proof Sketch of \cref{prop:prop2}}
\label{sec:proof of prop2}

\noindent \textbf{Proposition 2.} \textit{(Bounds for  $\mathcal{L}(\Tilde{\boldsymbol{\theta}}^{(m)})$ and the violation of $g(\Tilde{\boldsymbol{\theta}}^{(m)})$~\cite{averagedDS-SJO-2009})
\label{prop2}
Let the dual sequence $\{\lambda^{(m)}\}^{\infty}_{m=1}$ be generated through \Cref{eq:dual-solution} and $\{\Tilde{\boldsymbol{\theta}}^{(m)}\}^{\infty}_{m=1}$ be the averages in \Cref{eq:primal-avg}. Under \cref{assump1}, we have
\begin{enumerate}
    \item An upper bound on the amount of constraint violation of $\Tilde{\boldsymbol{\theta}}^{(m)}$ that $\big\lVert\big[g(\Tilde{\boldsymbol{\theta}}^{(m)})\big]_+\big\rVert\leq \frac{\lambda^{(m)}}{m\eta_2}$.
    \item An upper bound on $\mathcal{L}(\Tilde{\boldsymbol{\theta}}^{(m)})$ that $\mathcal{L}(\Tilde{\boldsymbol{\theta}}^{(m)}) \leq f^*+\frac{(\lambda^{(0)})^2}{2m\eta_2}+\frac{\eta_2 L^2}{2}$, where $\big\lVert g(\Tilde{\boldsymbol{\theta}}^{(m)})\big\rVert<L$ and $L>0$.
    \item A lower bound $\mathcal{L}(\Tilde{\boldsymbol{\theta}}^{(m)}) \geq f^*-\lambda^*\cdot\big\lVert\big[g(\Tilde{\boldsymbol{\theta}}^{(m)})\big]_+\big\rVert$.
\end{enumerate}
where $[u]_+$ denotes the projection of $[u]$ on the nonnegative orthant. $f^*$ is the optimal solution of \Cref{eq:opt-problem} and $\lambda^\ast$ denotes the optimal value of the dual variable.}  

\begin{proof}
1. According to \Cref{eq:dual-solution}, we have 
\begin{align*}
    \lambda^{(m)} 
    &\geq \lambda^{(m-1)}+\eta_2\cdot\Big(g(\Tilde{\boldsymbol{\theta}}^{(m)})-\gamma-\delta\lambda^{(m-1)}\Big).
\end{align*}
Under \cref{assump1}, $g(\boldsymbol{\theta})$ is convex, we have
\begin{align*}
    g(\Tilde{\boldsymbol{\theta}}^{(m)})
    &\leq \frac{1}{m}\sum_{i=1}^{m-1} g(\boldsymbol{\theta}^{(i)})\\
    &= \frac{1}{m\eta_2} \sum_{i=1}^{m-1} \eta_2 g(\boldsymbol{\theta}^{(i)})\\
    &\leq \frac{1}{m\eta_2} (\lambda^{(m)}-\lambda^{(0)})\\
    &\leq \frac{\lambda^{(m)}}{m\eta_2}, \quad\forall m\geq 1.
\end{align*}
Since $\lambda^{(m)}\geq 0$, we derive $||[g(\Tilde{\boldsymbol{\theta}}^{(m)})]_+||\leq \frac{\lambda^{(m)}}{m\eta_2}$.

2. Under \cref{assump1} and \Cref{eq:primal-solution,eq:dual-solution}, we have $q^*=f^*$. Together with the condition that $f(\boldsymbol{\theta})$ is convex and $\boldsymbol{\theta}\in\Theta$, we have
\begin{align*}
    &\mathcal{L}(\Tilde{\boldsymbol{\theta}}^{(m)})\\
    \leq &\frac{1}{m}\sum_{i=0}^{m-1} \mathcal{L}(\boldsymbol{\theta}^{(i)}) \\
    = &\frac{1}{m}\sum_{i=0}^{m-1} \Big(\mathcal{L}(\boldsymbol{\theta}^{(i)})+\lambda^{(i)} g(\Tilde{\boldsymbol{\theta}}^{(i+1)})-\lambda^{(i)}g(\Tilde{\boldsymbol{\theta}}^{(i+1)}) \Big)\\
    = &\frac{1}{m}\sum_{i=0}^{m-1} \Big(\mathcal{L}(\boldsymbol{\theta}^{(i)})+\lambda^{(i)} g(\Tilde{\boldsymbol{\theta}}^{(i+1)})\Big)- \frac{1}{m}\sum_{i=0}^{(m-1)}\lambda^{(i)} g(\Tilde{\boldsymbol{\theta}}^{(i+1)}) \\
    \leq &q^* -\frac{1}{m}\sum_{i=0}^{m-1}\lambda^{(i)}g(\Tilde{\boldsymbol{\theta}}^{(i+1)}).
\end{align*}
From \Cref{eq:dual-solution}, we have
\begin{align*}
    (\lambda^{(i+1)})^2
    &= \Big(\Big[\lambda^{(i)}+\eta_2\Big(g(\Tilde{\boldsymbol{\theta}}^{(i+1)})-\gamma-\delta\lambda^{(i)}\Big)\Big]_+\Big)^2\\
    &\leq \Big(\lambda^{(i)}+\eta_2 g(\Tilde{\boldsymbol{\theta}}^{(i+1)})\Big)^2\\
    &\leq (\lambda^{(i)})^2+\quad2\eta_2\lambda^{(i)} g(\Tilde{\boldsymbol{\theta}}^{(i+1)})\\
    &\quad+\Big(\eta_2||g(\Tilde{\boldsymbol{\theta}}^{(i+1)})||\Big)^2\\
    -\lambda^{(i)} g(\Tilde{\boldsymbol{\theta}}^{(i+1)}) & \leq \frac{(\lambda^{(i)})^2-(\lambda^{(i+1)})^2+\Big(\eta_2||g(\Tilde{\boldsymbol{\theta}}^{(i+1)})||\Big)^2}{2\eta_2}.
\end{align*}
Taking $-\lambda^{(i)} g(\Tilde{\boldsymbol{\theta}}^{(i+1)})$ back to $\mathcal{L}(\Tilde{\boldsymbol{\theta}}^{(m)})$, we have
\begin{align*}
    &\mathcal{L}(\Tilde{\boldsymbol{\theta}}^{(m)})\\
    \leq &q^*+\frac{1}{m}\sum_{i=0}^{m-1} \frac{(\lambda^{(i)})^2-(\lambda^{(i+1)})^2+\Big(\eta_2||g(\Tilde{\boldsymbol{\theta}}^{(i+1)})||\Big)^2}{2\eta_2}\\
    = &q^* + \frac{1}{m}\sum_{i=0}^{m-1} \frac{(\lambda^{(i)})^2-(\lambda^{(i+1)})^2}{2\eta_2} + \frac{1}{m}\sum_{i=0}^{m-1}\frac{\Big(\eta_2||g(\Tilde{\boldsymbol{\theta}}^{(i+1)})||\Big)^2}{2\eta_2}\\
    = &q^* + \frac{(\lambda^{(0)})^2-(\lambda^{(m)})^2}{2m\eta_2} + \frac{\eta_2}{2m}\sum_{i=0}^{m-1}||g(\Tilde{\boldsymbol{\theta}}^{(i+1)})||^2\\
    \leq &f^* + \frac{(\lambda^{(0)})^2}{2m\eta_2} + \frac{\eta_2 L^2}{2}.
\end{align*}

3. By definition, $\forall \boldsymbol{\theta}\in\Theta$, we have
\begin{align*}
    \mathcal{L}(\boldsymbol{\theta})+\lambda^*\cdot g(\boldsymbol{\theta})
    \geq \mathcal{L}(\boldsymbol{\theta}^*)+\lambda^*\cdot g(\boldsymbol{\theta}^*)
    = q(\lambda^*).
\end{align*}
Since $\Tilde{\boldsymbol{\theta}}\in\Theta$, $\forall m\geq 1$, we have
\begin{align*}
    \mathcal{L}(\Tilde{\boldsymbol{\theta}}^{(m)})
    &= \mathcal{L}(\Tilde{\boldsymbol{\theta}}^{(m)}) + \lambda^*\cdot g(\Tilde{\boldsymbol{\theta}}^{(m)}) - \lambda^*\cdot g(\Tilde{\boldsymbol{\theta}}^{(m)})\\
    &\geq q(\lambda^*) - \lambda^*\cdot g(\Tilde{\boldsymbol{\theta}}^{(m)})\\
    &\geq q(\lambda^*) - \lambda^*\cdot \big[g(\Tilde{\boldsymbol{\theta}}^{(m)})\big]_+\\
    &\geq f^*-\lambda^*\cdot \big\lVert \big[g(\Tilde{\boldsymbol{\theta}}^{(m)})\big]_+\big\rVert.
\end{align*}
\end{proof}

\section{Proof Sketch of \cref{prop:prop3}}
\label{sec:proof of prop3}
The Beta loss in \Cref{eq:beta-loss} is equivalent to the loss function used in DEAR when each actor is associated with one action only. The Evidential Neural Network (ENN), which was initially introduced in~\cite{sensoy-2018-nips} and further adopted by DEAR in open-set action recognition, is limited to classifying an actor associated with only one action. The key idea of ENN is to replace the output of a classification network with the parameters $\{\alpha_i\}_{i=1}^K$ of $K$ Dirichlet densities. To detect a novel actor with multiple actions, in this work, we modify ENN by estimating parameter pairs $\{\alpha_i,\beta_i\}_{i=1}^K$ of $K$ Beta distributions. \Cref{prop:prop3} shows that the Beta loss of an actor proposed in \Cref{eq:beta-loss-1-actor} is equivalent to the loss function in DEAR when $K=1$.

\begin{proposition}
\label{prop:prop3}
Denote $\mathcal{L}'_j(\boldsymbol{\theta})$ as the loss function introduced in~\cite{sensoy-2018-nips}, where
\begin{align}
\label{eq:enn-loss}
    \mathcal{L}'_j(\boldsymbol{\theta}) = \sum_{i=1}^{K'} y_{ij} \Big(\psi(\sum_{i=1}^{K'} \alpha_{ij})-\phi(\alpha_{ij})  \Big),
\end{align}
where $K'$ is the total number of classes in a multi-class classification task. Denote $\mathcal{L}_j(\boldsymbol{\theta})$ is the loss function proposed in \Cref{eq:beta-loss-1-actor}. We have $\mathcal{L}'_j(\boldsymbol{\theta})=\mathcal{L}_j(\boldsymbol{\theta})$ when $K=1$ (\ie, $K'=2$).
\end{proposition}
\begin{proof}
When $K=1$,
\begin{align*}
    \mathcal{L}_j(\boldsymbol{\theta}) = \sum\nolimits_{i=1}^1 \int\Big[\textbf{BCE}(y_{ij},p_{ij})\Big]\textbf{Beta}(p_{ij};\alpha_{ij},\beta_{ij})dp_{ij}.
\end{align*}
To simplify, we omit the subscript $i$ and rewrite
\begin{align*}
    \small
    &\mathcal{L}_j(\boldsymbol{\theta}) 
    =\int\Big[\textbf{BCE}(y_{j},p_{j})\Big]\textbf{Beta}(p_{j};\alpha_{j},\beta_{j})dp_{j}\\
    &=y_{j}\Big(\psi(\alpha_{j}+\beta_{j})-\psi(\alpha_{j})\Big)+(1-y_{j})\Big(\psi(\alpha_{j}+\beta_{j})-\psi(\beta_{j})\Big).
\end{align*}
As for $\mathcal{L}'_j(\boldsymbol{\theta})$, $K=1$ indicates binary classification which refers $K'=2$.
\begin{align*}
    \mathcal{L}'_j(\boldsymbol{\theta}) 
    &=\int\Big[\textbf{CE}(y_{j},p_{j})\Big]\textbf{Dir}(p_{j};\alpha_{j},\beta_{j})dp_{j}\\
    &=\sum_{i=1}^2 y_{ij}\Big(\psi(\sum_{i=1}^2 \alpha_{ij})-\psi(\alpha_{ij}) \Big).
\end{align*}
We complete the proof by setting $\alpha_{1j}=\alpha_j$ and $\alpha_{2j}=\beta_{j}$.
\end{proof}

\section{Detailed Ablation Study}
In Table \ref{tab:ablation-compinent_m345}, we provide a detailed ablation study on the AVA dataset~\cite{gu-2018-cvpr-ava} to explore the contributions of different components in our framework. Along with the primal-dual updating step $m=2,3,4,5$, the performance is slightly improved but nearly saturated when $m=2$. 

\section{Full Experimental Results}
\label{sec:additional results}
As presented in Table \ref{tab:baseline_AVA_full} and Table \ref{tab:baseline_Charades_full}, we report the results of compared methods by using all four novelty score estimation mechanism. It is worth mentioning that other state-of-the-arts perform also well by using the belief based score in terms of some metrics.

Furthermore, we show more visual results of compared methods in Figure \ref{fig:visualization-appendix-AVA} and Figure \ref{fig:visualization-appendix-Charades}. It can be seen that our method performs better than other methods on two datasets for single/multi-actor settings.

\begin{figure*}[t!]
    \centering
    \includegraphics[width=\linewidth]{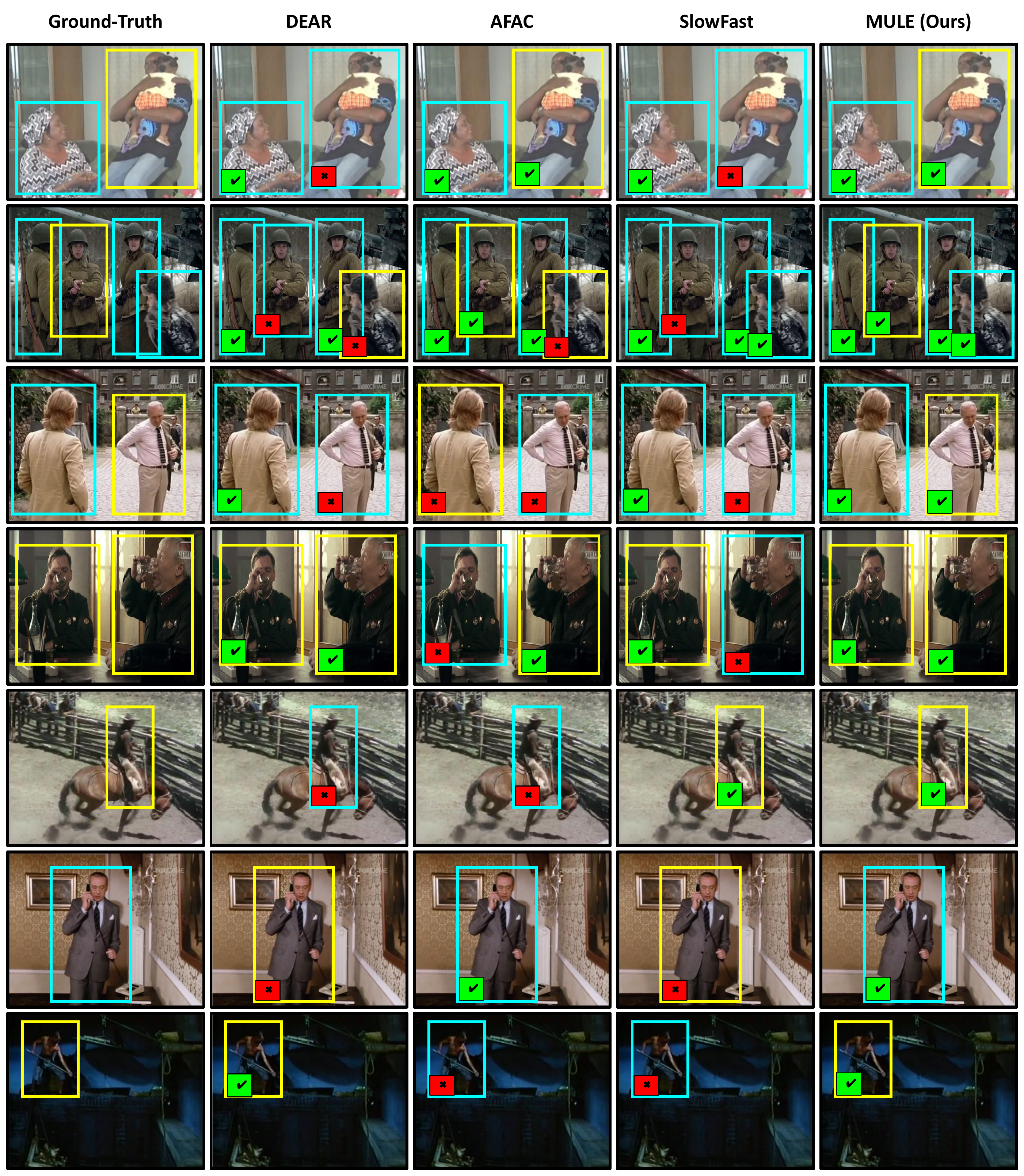}
    \caption{Visual comparison with our method and state-of-the-art on AVA~\cite{gu-2018-cvpr-ava}. \textcolor{Cyan}{Cyan} and \textcolor{Goldenrod!80!black}{yellow} boxes denote the predictions of actors with known and novel actions, respectively. \textcolor{Green}{\cmark ~marks} and \textcolor{Red}{\xmark~marks} indicate correct and false predictions, respectively.}
    \label{fig:visualization-appendix-AVA}
\end{figure*}

\begin{figure*}[t!]
    \centering
    \includegraphics[width=\linewidth]{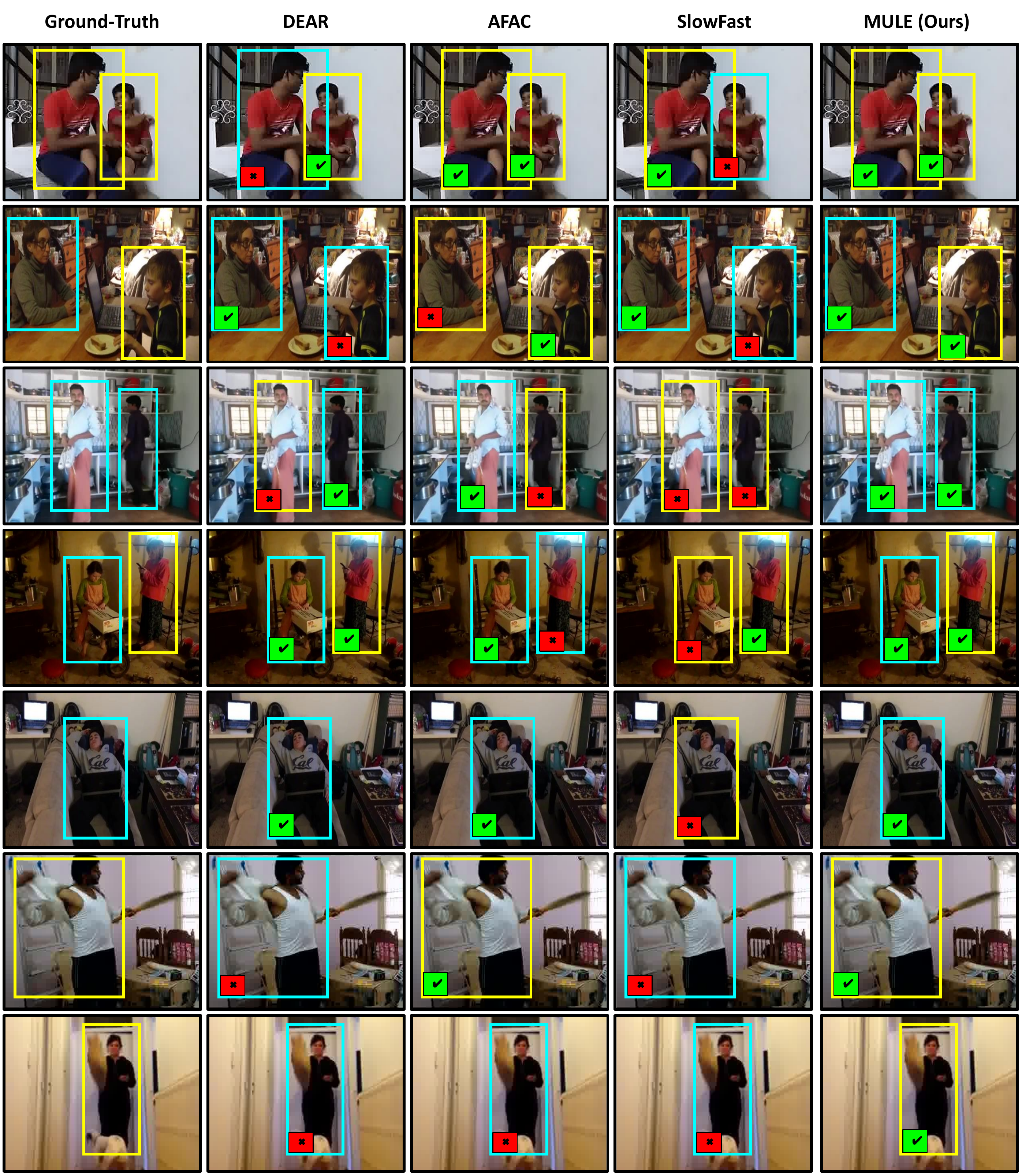}
    \caption{Visual comparison with our method and state-of-the-art on Charades~\cite{sigurdsson2016hollywood_Charades}. \textcolor{Cyan}{Cyan} and \textcolor{Goldenrod!80!black}{yellow} boxes denote the predictions of actors with known and novel actions, respectively. \textcolor{Green}{\cmark ~marks} and \textcolor{Red}{\xmark~marks} indicate correct and false predictions, respectively.}
    \label{fig:visualization-appendix-Charades}
\end{figure*}

\end{document}